\documentclass{article}

\usepackage{natbib}

\usepackage{url}            % simple URL typesetting
\usepackage{booktabs}       % professional-quality tables
\usepackage{amsfonts}       % blackboard math symbols

\usepackage{float}
\usepackage{amsmath}
\usepackage{amssymb}
\usepackage{amsthm}

\floatstyle{ruled}
\newfloat{algorithm}{tbp}{loa}
\providecommand{\algorithmname}{Algorithm}
\floatname{algorithm}{\protect\algorithmname}

\theoremstyle{plain}
\newtheorem{thm}{\protect\theoremname}
\theoremstyle{definition}
\newtheorem{defn}[thm]{\protect\definitionname}
\theoremstyle{plain}
\newtheorem{lem}[thm]{\protect\lemmaname}
\theoremstyle{plain}
\newtheorem{cor}[thm]{\protect\corollaryname}
\ifx\proof\undefined
\newenvironment{proof}[1][\protect\proofname]{\par
\normalfont\topsep6\p@\@plus6\p@\relax
\trivlist
\itemindent\parindent
\item[\hskip\labelsep\scshape #1]\ignorespaces
}{%
\endtrivlist\@endpefalse
}
\providecommand{\proofname}{Proof}
\fi

\usepackage{algorithmic}
\usepackage{enumitem}

\setlist[itemize]{noitemsep}

\providecommand{\corollaryname}{Corollary}
\providecommand{\definitionname}{Definition}
\providecommand{\lemmaname}{Lemma}
\providecommand{\theoremname}{Theorem}

\usepackage{enumitem}
\setlist[itemize]{noitemsep}

\usepackage{authblk}

\title{A Non-convex One-Pass Framework for Generalized Factorization Machine and Rank-One Matrix Sensing}

\author[1]{Ming Lin\thanks{linmin@umich.edu}}
\author[1]{Jieping Ye\thanks{jpye@umich.edu}}
\affil[1]{Department of Computational Medicine and Bioinformatic\\
  University of Michigan, Ann Arbor, MI 48109}

\date{August 21, 2016}

\begin{document}

\maketitle

\begin{abstract}
We develop an efficient alternating framework for learning a generalized version of Factorization
Machine (gFM) on steaming data with provable guarantees. When the instances
are sampled from $d$ dimensional random Gaussian vectors and the target second order coefficient matrix in
gFM is of rank $k$, our algorithm converges linearly, achieves $O(\epsilon)$
recovery error after retrieving $O(k^{3}d\log(1/\epsilon))$ training
instances, consumes $O(kd)$ memory in one-pass of dataset and only
requires matrix-vector product operations in each iteration. The key
ingredient of our framework is a construction of an estimation sequence
	endowed with a so-called Conditionally Independent RIP condition (CI-RIP).
As special cases of gFM, our framework can be applied to symmetric
or asymmetric rank-one matrix sensing problems, such as inductive
matrix completion and phase retrieval.
\end{abstract}

\section{Introduction}

Linear models are one of the foundations of modern machine learning
due to their strong learning guarantees and efficient solvers \citep{koltchinskii_oracle_2011}.
Conventionally linear models only consider the first order information
of the input feature which limits their capacity in non-linear problems.
Among various efforts extending linear models to the non-linear domain,
the Factorization Machine \citep{rendle_factorization_2010} (FM)
captures the second order information by modeling the pairwise feature
interaction in regression under low-rank constraints. FMs have been
found successful in many applications, such as recommendation systems
\citep{rendle_fast_2011} and text retrieval \citep{hong_co-factorization_2013}.
In this paper, we consider a generalized version of FM called gFM which removes several redundant constraints in the original FM 
such as positive semi-definite and zero-diagonal, leading to a more general model without sacrificing its learning ability.
From theoretical side, the gFM includes rank-one matrix sensing \citep{zhong_efficient_2015,chen_exact_2015,cai_rop_2015,kueng_low_2014}
as a special case, where the latter one has been studied widely in
context such as inductive matrix completion \citep{jain_provable_2013}
and phase retrieval \citep{candes_phase_2011}.

Despite of the popularity of FMs in industry, there is rare theoretical
study of learning guarantees for FMs. One of the main challenges in
developing a provable FM algorithm is to handle its symmetric rank-one matrix
sensing operator. For conventional matrix sensing problems where the matrix sensing operator is
RIP, there are several alternating
methods with provable guarantees \citep{moritz_hardt_understanding_2013,jain_low-rank_2013,hardt_fast_2014,zhao_nonconvex_2015,zhao_nonconvex_2015-1}.
However, for a symmetric rank-one matrix sensing operator, the RIP condition doesn't hold trivially which
turns out to be the main difficulty in designing efficient provable FM solvers.

In rank-one matrix sensing, when the sensing operator is asymmetric,
the problem is also known as inductive matrix completion which can
be solved via alternating minimization with a global linear convergence
rate \citep{jain_provable_2013,zhong_efficient_2015}. For symmetric
rank-one matrix sensing operators, we are not aware of any efficient
solver by the time of writing this paper. In a special case when the
target matrix is of rank one, the problem is called ``phase retrieval''
whose convex solver is first proposed by \citet{candes_phase_2011}
then alternating methods are provided in \citep{lee_near_2013,netrapalli_phase_2013}.
While the target matrix is of rank $k>1$ , only convex methods minimizing
the trace norm have been proposed recently, which are computationally
expensive \citep{kueng_low_2014,cai_rop_2015,chen_exact_2015,davenport_overview_2016}.

Despite of the above fundamental challenges, extending rank-one matrix
sensing algorithm to gFM itself is difficult. Please refer to Section
\ref{sub:FMs-and-Rank-One-matrix-sensing} for an in-depth discussion.
The main difficulty is due to the first order term in the gFM formulation,
which cannot be trivially converted to a standard matrix sensing problem.

In this paper, we develop a unified theoretical framework and an efficient
solver for generalized Factorization Machine and its special cases such as rank-one matrix sensing,
either symmetric or asymmetric. The key ingredient is to show that
the sensing operator in gFM satisfies a so-called Conditionally Independent
RIP condition (CI-RIP, see Definition \ref{def:RIP}) . Then we can
construct an estimation sequence via noisy power iteration \citep{hardt_noisy_2013}.
Unlike previous approaches, our method does not require alternating
minimization or choosing the step-size as in alternating gradient
descent. The proposed method works on steaming data, converges linearly
and has $O(kd)$ space complexity for a $d$-dimension rank-$k$ gFM
model. The solver achieves $O(\epsilon)$ recovery error after retrieving
$O(k^{3}d\log(1/\epsilon))$ training instances.

The remainder of this paper is organized as following. In Section
2, we introduce necessary notation and background of gFM. Subsection
\ref{sub:FMs-and-Rank-One-matrix-sensing} investigates several fundamental
challenges in depth. Section 3 presents our algorithm, called
One-Pass gFM, followed by its theoretical guarantees.
Our analysis framework is presented in Section 4. Section 5 concludes
this paper.

\section{Generalized Factorization Machine (gFM)}

In this section, we first introduce necessary notation and background
of FM and its generalized version gFM. Then in Subsection \ref{sub:FMs-and-Rank-One-matrix-sensing},
we reveal the connection between gFM and rank-one matrix sensing
followed by several fundamental challenges encountered when applying
frameworks of rank-one matrix sensing to gFM.

The FM predicts the labels of instances by not only their features
but also high order interactions between features. In the following,
we focus on the second order FM due to its popularity. Suppose we
are given $N$ training instances $\boldsymbol{x}_{i}\in\mathbb{R}^{d}$
independently and identically (I.I.D.) sampled from the standard Gaussian
distribution and so are their associated labels $y_{i}\in\mathbb{R}$.
Denote the feature matrix $X=[\boldsymbol{x}_{1},\boldsymbol{x}_{2},\cdots,\boldsymbol{x}_{n}]\in\mathbb{R}^{d\times n}$
and the label vector $\boldsymbol{y}=[y_{1},y_{2},\cdots,y_{n}]{}^{\top}\in\mathbb{R}^{n}$
. In second order FM, $y_{i}$ is assumed to be generated from a
target vector $\boldsymbol{w}^{*}\in\mathbb{R}^{d}$ and a target
rank $k$ matrix $M^{*}\in\mathbb{R}^{d\times d}$ satisfying
\begin{align}
y_{i}= & \boldsymbol{x}_{i}{}^{\top}\boldsymbol{w}^{*}+\boldsymbol{x}_{i}{}^{\top}M^{*}\boldsymbol{x}_{i}+\xi_{i}\label{eq:y_i=00003Dx_iw*+x_iM*x_i+e_i}
\end{align}
 where $\xi_{i}$ is a random subgaussian noise with proxy variance
$\xi^{2}$ . It is often more convenient to write Eq. (\ref{eq:y_i=00003Dx_iw*+x_iM*x_i+e_i})
in matrix form. Denote the linear operator $\mathcal{A}:\mathbb{R}^{d\times d}\rightarrow\mathbb{R}^{n}$
as $\mathcal{A}(M)\triangleq[\left\langle A_{1},M\right\rangle ,\left\langle A_{2},M\right\rangle ,\cdots,\left\langle A_{n},M\right\rangle ]{}^{\top}$
where $A_{i}=\boldsymbol{x}_{i}\boldsymbol{x}_{i}{}^{\top}$ . Then
Eq. (\ref{eq:y_i=00003Dx_iw*+x_iM*x_i+e_i}) has a compact form: 
\begin{align}
\boldsymbol{y}=X{}^{\top}\boldsymbol{w}^{*}+ & \mathcal{A}(M^{*})+\boldsymbol{\xi}\ .\label{eq:operator-form-of-FM}
\end{align}

The FM model given by Eq. (\ref{eq:operator-form-of-FM}) consists
of two components: the first order component $X{}^{\top}\boldsymbol{w}^{*}$
and the second order component $\mathcal{A}(M^{*})$. The component
$\mathcal{A}(M^{*})$ is a symmetric rank-one Gaussian measurement
since $\mathcal{A}_{i}(M)=\boldsymbol{x}_{i}{}^{\top}M\boldsymbol{x}_{i}$
where the left/right design vectors ($\boldsymbol{x}_{i}$ and $\boldsymbol{x}_{i}{}^{\top}$)
are identical. The original FM requires that $M^*$ should be positive semi-definite and the diagonal
elements of $M^*$ should be zero. However our analysis shows that both constraints are redundant for 
learning Eq. \ref{eq:operator-form-of-FM}. Therefore in this paper we consider a generalized version 
of FM which we call gFM where $M^*$ is only required to be symmetric and low rank. To make the recovery
of $M^{*}$ well defined, it is necessary to assume
$M^{*}$ to be symmetric. Indeed for any asymmetric matrix $M^{*}$,
there is always a symmetric matrix $M_{\mathrm{sym}}^{*}=(M^{*}+M^{*}{}^{\top})/2$
such that $\mathcal{A}(M^{*})=\mathcal{A}(M_{\mathrm{sym}}^{*})$
thus the symmetric constraint does not affect the model. Another standard
assumption in rank-one matrix sensing is that the rank of $M^{*}$
should be no more than $k$ for $k\ll d$. When $\boldsymbol{w}^{*}=0$, gFM is equal to the symmetric rank-one matrix sensing problem. 
Recent researches have proposed several convex programming methods based on the trace norm minimization to recover $M^{*}$ with a sampling complexity on order
of $O(k^{3}d)$ \citep{candes_phase_2011,cai_rop_2015,kueng_low_2014,chen_exact_2015,zhong_efficient_2015}.
Some authors also call gFM as second order polynomial network \citep{blondel_polynomial_2016}.

When $d$ is much larger than $k$, the convex programming on the trace
norm or nuclear norm of $M^{*}$ becomes difficult since $M^{*}$ can be a
 $d\times d$ dense matrix. Although modern convex solvers can scale to large $d$ with reasonable computational cost, 
 a more popular strategy to efficiently estimate $\boldsymbol{w}^{*}$
and $M^{*}$ is to decompose $M^{*}$ as $UV{}^{\top}$ for some $U,V\in\mathbb{R}^{d\times k}$,
then alternatively update $\boldsymbol{w}$, $U,V$ to minimize the
empirical loss function 
\begin{align}
\min_{\boldsymbol{w},U,V}\  & \frac{1}{2N}\|\boldsymbol{y}-X{}^{\top}\boldsymbol{w}-\mathcal{A}(UV{}^{\top})\|_{2}^{2}\ .\label{eq:least-square-loss-FM}
\end{align}
 The loss function in Eq. (\ref{eq:least-square-loss-FM}) is non-convex.
It is even unclear whether an estimator of the optimal solution $\{\boldsymbol{w}^{*},M^{*}\}$
of Eq. (\ref{eq:least-square-loss-FM}) with a polynomial time complexity
exists or not.

In our analysis, we denote $M+O(\epsilon)$ as a matrix $M$ plus
a perturbation matrix whose spectral norm is bounded by $\epsilon$.
We use $\|\cdot\|_{2}$ , $\|\cdot\|_{F}$ , $\|\cdot\|_{*}$ to denote
the matrix spectral norm, Frobenius norm and nuclear norm respectively.
To abbreviate the high probability bound, we denote $C=\mathrm{polylog(d,n,T,1/\eta)}$
to be a constant polynomial logarithmic  in $\{d,n,T,1/\eta\}$. The
eigenvalue decomposition of $M^{*}$ is $M^{*}=U^{*}\Lambda^{*}U^{*}{}^{\top}$
where $U^{*}\in\mathbb{R}^{d\times k}$ is the top-$k$ eigenvectors
of $M^{*}$ and $\Lambda^{*}=\mathrm{diag}(\lambda_{1}^{*},\lambda_{2}^{*},\cdots,\lambda_{k}^{*})$
are the corresponding eigenvalues sorted by $|\lambda_i| \geq |\lambda_{i+1}|$. Let $\sigma_{i}^{*}=|\lambda_{i}^{*}|$
 denote the singular value of $M^*$ and $\sigma_{i}\{M\}$ be the $i$-th largest singular value of $M$.
$U_{\perp}^{*}$ denotes an matrix whose columns are the orthogonal basis of the complementary subspace of $U^{*}$.

\subsection{gFM and Rank-One Matrix Sensing \label{sub:FMs-and-Rank-One-matrix-sensing}}

When $\boldsymbol{w}^{*}=0$ in Eq. (\ref{eq:y_i=00003Dx_iw*+x_iM*x_i+e_i}),
the gFM becomes the symmetric rank-one matrix sensing problem. While the
recovery ability of rank-one matrix sensing is somehow provable recently
despite of the computational issue, it is not the case for gFM. It
is therefore important to discuss the differences between gFM and
rank-one matrix sensing to give us a better understanding of the fundamental
barriers in developing provable gFM algorithm.

In the rank-one matrix sensing problem, a relaxed setting is to assume
that the sensing operator is asymmetric, which is defined by $\mathcal{A}_{i}^{\mathrm{asy}}(M)=\boldsymbol{u}_{i}{}^{\top}M\boldsymbol{v}_{i}$
where $\boldsymbol{u}_{i}$ and $\boldsymbol{v_{i}}$ are independent
random vectors. Under this setting, the recovery ability of alternating
methods is provable \citep{jain_provable_2013}. However, existing
analyses cannot be generalized to their symmetric counterpart, since
$\boldsymbol{u}_{i}$ and $\boldsymbol{v}_{i}$ are not allowed to
be dependent in these frameworks. For example, the sensing operator
$\mathcal{A}^{\mathrm{asy}}(\cdot)$ is unbiased ( $E\mathcal{A}^{\mathrm{asy}}(\cdot)=0$)
but the symmetric sensing operator is clearly not \citep{cai_rop_2015}.
Therefore, the asymmetric setting oversimplifies the problem and loses
important structure information which is critical to gFM.

As for the symmetric rank-one matrix sensing operator, the state-of-the-art
estimator is based on the trace norm convex optimization \citep{tropp_convex_2014, chen_exact_2015,cai_rop_2015},
which is computationally expensive. When $\boldsymbol{w}^{*}\not=\boldsymbol{0}$, the
gFM has an extra perturbation term $X{}^{\top}\boldsymbol{w}^{*}$
. This first order perturbation term turns out to be a fundamental challenge in theoretical analysis. One might attempt
to merge $\boldsymbol{w}^{*}$ into $M^{*}$ in order to convert gFM
as a rank $(k+1)$ matrix sensing problem. For example, one may extend
the feature $\hat{\boldsymbol{x}}_{i}\triangleq[\boldsymbol{x}_{i},1]{}^{\top}$
and the matrix $\hat{M}^{*}=[M^{*};\boldsymbol{w}^{*}{}^{\top}]\in\mathbb{R}^{(d+1)\times d}$.
However, after this simple extension, the sensing operator becomes
$\hat{\mathcal{A}}(M^{*})=\hat{\boldsymbol{x}}_{i}{}^{\top}\hat{M}^{*}\boldsymbol{x}_{i}$.
It is no longer symmetric. The left/right design vector is neither
independent nor identical. Especially, not all dimensions of $\hat{\boldsymbol{x}}_{i}$
are random variables. According to the above discussion, the conditions
to guarantee the success of rank-one matrix sensing do not hold after
feature extension and all the mentioned analyses cannot be directly applied.

\section{One-Pass gFM\label{sec:One-Pass-Factorization-Machine}}

In this section, we present the proposed algorithm, called One-Pass gFM followed by its theoretical guarantees. We will focus
on the intuition of our algorithm. A rigorous theoretical analysis
 is presented in the next section.

The One-Pass gFM is a mini-batch algorithm. In each mini-batch, it processes $n$
training instances and then alternatively updates parameters. The
iteration will continue until $T$ mini-batch updates.  Since gFM deals with a non-convex learning problem, 
the conventional gradient descent framework hardly works
to show the global convergence. Instead, our method is based on a construction of an estimation sequence.
Intuitively, when $\boldsymbol{w}^*=\boldsymbol{0}$, we will show in the next section that
$\frac{1}{n}\mathcal{A}'\mathcal{A}(M) \approx 2M + \mathrm{tr}(M)I$ and 
$\mathrm{tr}(M) \approx \frac{1}{n} \boldsymbol{1}^\top \mathcal{A}(M)$. Since $\boldsymbol{y}\approx\mathcal{A}(M^*)$,
we can estimate $M^*$ via $\frac{1}{2n}\mathcal{A}'(\boldsymbol{y})-\frac{1}{n} \boldsymbol{1}^\top \boldsymbol{y}I$.
But this simple construction cannot generate a convergent estimation sequence since the perturbation terms
in the above approximate equalities cannot be reduced along iterations. To overcome this problem, we replace $\mathcal{A}(M^*)$ with 
$\mathcal{A}(M^*-M^{(t)})$ in our construction. Then the perturbation terms will be on order of $O(\|M^*-M^{(t)}\|_2)$.
When $\boldsymbol{w}^*\not=\boldsymbol{0}$, we can apply a similar trick to construct its estimation sequence via 
the second and the third order moments of $X$. Algorithm \ref{alg:OPFMs} gives a step-by-step description of our algorithm\footnote{Implementation is available from \url{https://minglin-home.github.io/}}.

\begin{algorithm}
\begin{algorithmic}[1]

\REQUIRE The mini-batch size $n$, number of total mini-batch update
$T$, training instances $X=[\boldsymbol{x}_{1},\boldsymbol{x}_{2},\cdots\boldsymbol{x}_{nT}\}$,
$\boldsymbol{y}=[y_{1},y_{2},\cdots,y_{nT}]{}^{\top}$, desired rank
$k\geq1$.

\ENSURE $\boldsymbol{w}^{(T)},U^{(T)},V^{(T)}$.

\STATE Define $M^{(t)}\triangleq(U^{(t)}V^{(t)}{}^{\top}+V^{(t)}U^{(t)}{}^{\top})/2$
, $H_{1}^{(t)}\triangleq\frac{1}{2n}\mathcal{A}'(\boldsymbol{y}-\mathcal{A}(M^{(t)}) - X^{(t)}{}^{\top}\boldsymbol{w}^{(t)})$
, $h_{2}^{(t)}\triangleq\frac{1}{n}\boldsymbol{1}{}^{\top}(\boldsymbol{y}-\mathcal{A}(M^{(t)})-X^{(t)}{}^{\top}\boldsymbol{w}^{(t)})$
	, $\boldsymbol{h}_{3}^{(t)}\triangleq\frac{1}{n}X^{(t)} (\boldsymbol{y}-\mathcal{A}(M^{(t)})-X^{(t)}{}^{\top}\boldsymbol{w}^{(t)})$
.

\STATE Initialize: $\boldsymbol{w}^{(0)}=\boldsymbol{0}$, $V^{(0)}=0$.
$U^{(0)}=\mathrm{SVD}(H_{1}^{(0)}-\frac{1}{2}h_{2}^{(0)}I,k)$, that
is, the top-$k$ left singular vectors.

\FOR{$t=1,2,\cdots,T$}

\STATE Retrieve $n$ training instances $X^{(t)}=[\boldsymbol{x}_{(t-1)n+1},\cdots,\boldsymbol{x}_{(t-1)n+n}]$
	. Define $\mathcal{A}(M)\triangleq [ X_{i}^{(t)}{}^{\top}MX_{i}^{(t)}]_{i=1}^{n}$.

\STATE $\hat{U}^{(t)}=(H_{1}^{(t-1)}-\frac{1}{2}h_{2}^{(t-1)}I+M^{(t-1)}{}^{\top})U^{(t-1)}$
.

\STATE Orthogonalize $\hat{U}^{(t)}$ via QR decomposition: $U^{(t)}=\mathrm{QR}\left(\hat{U}^{(t)}\right)$ .

	\STATE $\boldsymbol{w}^{(t)}=\boldsymbol{h}_{3}^{(t-1)}+\boldsymbol{w}^{(t-1)}$ .

\STATE $V^{(t)}=(H_{1}^{(t-1)}-\frac{1}{2}h_{2}^{(t-1)}I+M^{(t-1)})U^{(t)}$

\ENDFOR

\STATE \textbf{Output:} $\boldsymbol{w}^{(T)},U^{(T)},V^{(T)}$ .
\end{algorithmic}

	\protect\caption{One-Pass gFM }
\label{alg:OPFMs}
\end{algorithm}

In Algorithm \ref{alg:OPFMs}, we only need to store $\boldsymbol{w}^{(t)}\in\mathbb{R}^{d}$,
$U^{(t)},V^{(t)}\in\mathbb{R}^{d\times k}$. Therefore the space complexity
is $O(d+kd)$. The auxiliary variables $M^{(t)},H_{1}^{(t)},h_{2}^{(t)},\boldsymbol{h}_{3}^{(t)}$
can be implicitly presented by $\boldsymbol{w}^{(t)},U^{(t)},V^{(t)}$.
In each mini-batch updating, we only need matrix-vector product operations which
can be efficiently implemented on many computation architectures. We
use truncated SVD to initialize gFM, a standard initialization step
in matrix sensing. We do not require this step to be computed exactly
but up to an accuracy of $O(\delta)$ where $\delta$ is the RIP constant.  The QR step on line 6 requires $O(k^{2}d)$
 operations. Compared with SVD which requires $O(kd^{2})$ operations, the QR step
is much more efficient when $d\gg k$. Algorithm \ref{alg:OPFMs} retrieves instances streamingly,
a favorable behavior on systems with high speed cache. Finally, we
export $\boldsymbol{w}^{(T)},U^{(T)},V^{(T)}$ as our estimation of
$\boldsymbol{w}^{*}\approx\boldsymbol{w}^{(T)}$ and $M^{*}\approx U^{(T)}V^{(T)}{}^{\top}$.

Our main theoretical result is presented in the following theorem, which gives
the convergence rate of recovery and sampling complexity of gFM when
$M^{*}$ is low rank and the noise $\boldsymbol{\xi}=\boldsymbol{0}$.
\begin{thm}
\label{thm:convergence-rate-of-OPFM} Suppose $\boldsymbol{x}_{i}$'s
are independently sampled from the standard Gaussian distribution. $M^{*}$ is a rank $k$ matrix. The noise
	$\boldsymbol{\xi}=\boldsymbol{0}$. Then with a probability at least $1-\eta$, there
exists a constant $C$ and a constant $\delta<1$ such that
\begin{align*}
\|\boldsymbol{w}^{*}-\boldsymbol{w}^{(t)}\|_{2}+\|M^{*}-M^{(t)}\|_{2}\leq & \delta^{t}(\|\boldsymbol{w}^{*}\|_{2}+\|M^{*}\|_{2})
\end{align*}
 provided $n\geq C(4\sqrt{5}\sigma_{1}^{*}/\sigma_{k}^{*}+3)^{2}k^{3}d/\delta^{2},\ \delta\leq\frac{(4\sqrt{5}\sigma_{1}^{*}/\sigma_{k}^{*}+3)\sigma_{k}^{*}}{4\sqrt{5}\sigma_{1}^{*}+3\sigma_{k}^{*}+4\sqrt{5}\|\boldsymbol{w}^{*}\|_{2}^{2}}\ .$ 
\end{thm}
Theorem \ref{thm:convergence-rate-of-OPFM} shows that $\{\boldsymbol{w}^{(t)},M^{(t)}\}$
will converge to $\{\boldsymbol{w}^{*},M^{*}\}$ linearly. The convergence
rate is controlled by $\delta$, whose value is on order of $O(1/\sqrt{n})$.
A small $\delta$ will result in a fast convergence rate but a large
sampling complexity. To reduce the sampling complexity, a large $\delta$
is preferred. The largest allowed $\delta$ is bounded by $O(1/(\|M^{*}\|_{2}+\|\boldsymbol{w}^{*}\|_{2}))$.
The sampling complexity is $O((\sigma_{1}^{*}/\sigma_{k}^{*})^{2}k^{3}d)$.
If $M^{*}$ is not well conditioned, it is possible to remove $(\sigma_{1}^{*}/\sigma_{k}^{*})^{2}$
in the sampling complexity by a procedure called ``soft-deflation''
\citep{jain_low-rank_2013,hardt_fast_2014}. By theorem \ref{thm:convergence-rate-of-OPFM},
gFM achieves $\epsilon$ recovery error after retrieving $nT=O(k^{3}d\log\left((\|\boldsymbol{w}^{*}\|_{2}+\|M^{*}\|_{2})/\epsilon\right))$
instances.

The noisy case where $M^{*}$ is not exactly low rank and $\xi>0$
is more intricate therefore we postpone it to Subsection \ref{sub:Noisy-Case}.
The main conclusion is similar to the noise-free case Theorem \ref{thm:convergence-rate-of-OPFM}
under a small noise assumption.

\section{Theoretical Analysis}

In this section, we give the sketch of our proof of Theorem \ref{thm:convergence-rate-of-OPFM}.
Omitted details are postponed to appendix.

From high level, our proof constructs an estimation sequence
$\{\widetilde{\boldsymbol{w}}^{(t)},\widetilde{M}^{(t)},\epsilon_{t}\}$
such that $\epsilon_{t}\rightarrow0$ and $\|\boldsymbol{w}^{*}-\widetilde{\boldsymbol{w}}^{(t)}\|_{2}+\|M^{*}-\widetilde{M}^{(t)}\|_{2}\leq\epsilon_{t}$
. In conventional matrix sensing, this construction is possible when
the sensing matrix satisfies the Restricted Isometric Property (RIP)
\citep{candes_exact_2009}:
\begin{defn}[$\ell_{2}$-norm RIP]
\label{def:RIP}  A sensing operator $\mathcal{A}$ is $\ell_{2}$-norm
$\delta_{k}$-RIP if for any rank $k$ matrix $M$,
\[
(1-\delta_{k})\|M\|_{F}\leq\frac{1}{n}\|\mathcal{A}(M)\|_{2}^{2}\leq(1+\delta_k)\|M\|_{F}\ .
\]

\end{defn}
When $\mathcal{A}$ is $\ell_2$-norm $\delta_k$-RIP for any rank $k$ matrix $M$,
$\mathcal{A}'\mathcal{A}$ is nearly isometric \citep{jain_low-rank_2012},
which implies $\|M-\mathcal{A}'\mathcal{A}(M)/n\|_{2}\leq\delta$.
Then we can construct our estimation sequence as following:
\begin{align*}
\widetilde{M}^{(t)}= & \frac{1}{n}\mathcal{A}'\mathcal{A}(M^{*}-\widetilde{M}^{(t-1)})+\widetilde{M}^{(t-1)}\ ,\ \widetilde{\boldsymbol{w}}^{(t)}=(I-\frac{1}{n}XX{}^{\top})(\boldsymbol{w}^{*}-\widetilde{\boldsymbol{w}}^{(t-1)})+\widetilde{\boldsymbol{w}}^{(t-1)}\ .
\end{align*}
However, in gFM and symmetric rank-one matrix sensing, the $\ell_2$-norm 
RIP condition cannot be satisfied with high probability \citep{cai_rop_2015}.
To establish an RIP-like condition for rank-one matrix sensing, several
variants have been proposed, such as the $\ell_{2}/\ell_{1}$-RIP condition
\citep{cai_rop_2015,chen_exact_2015}. The essential idea of these
variants is to replace the $\ell_{2}$-norm $\|\mathcal{A}(M)\|_{2}$
with $\ell_{1}$-norm $\|\mathcal{A}(M)\|_{1}$ then a similar norm
inequality can be established for all low rank matrix again. However,
even using these $\ell_{1}$-norm RIP variants, we are still
unable to design an efficient alternating algorithm. All these $\ell_{1}$-norm
RIP variants have to deal with trace norm programming problems.
In fact, it is impossible to construct an estimation sequence based
on $\ell_{1}$-norm RIP because we require $\ell_{2}$-norm bound
on $\mathcal{A}'\mathcal{A}$ during the construction.

A key ingredient of our framework is to propose a novel $\ell_{2}$-norm
RIP condition to overcome the above difficulty. The main technique
reason for the failure of conventional $\ell_{2}$-norm RIP is that
it tries to bound $\mathcal{A}'\mathcal{A}(M)$ over all rank $k$
matrices. This is too aggressive to be successful in rank-one matrix
sensing. Regarding to our estimation sequence, what we really need
is to make the RIP hold for current low rank matrix $M^{(t)}$. Once
we update our estimation $M^{(t+1)}$, we can regenerate a new sensing
operator independent of $M^{(t)}$ to avoid bounding $\mathcal{A}'\mathcal{A}$
over all rank $k$ matrices. To this end, we propose the Conditionally Independent
RIP (CI-RIP) condition.
\begin{defn}[CI-RIP]
 \label{def:CI-RIP-condition} A matrix sensing operator $\mathcal{A}$
is Conditionally Independent RIP with constant $\delta_{k}$, if for
a fixed rank $k$ matrix $M$, $\mathcal{A}$ is sampled independently
regarding to $M$ and satisfies
\begin{equation}
\|(I-\frac{1}{n}\mathcal{A}'\mathcal{A})M\|_{2}^{2}\leq\delta_{k}\ .\label{eq:CI-RIP}
\end{equation}

\end{defn}
An $\ell_{2}$-norm or $\ell_{1}$-norm RIP sensing operator is
naturally CI-RIP but the reverse is not true. In CI-RIP, $\mathcal{A}$
is no longer a fixed but random sensing operator independent of $M$.
In one-pass algorithm, this is achievable if we always retrieve new
instances to construct $\mathcal{A}$ in one mini-batch updating.
Usually Eq. (\ref{eq:CI-RIP}) doesn't hold in a batch method since
$M^{(t+1)}$ depends on $\mathcal{A}(M^{(t)})$. 

An asymmetric rank-one matrix sensing operator is clearly CI-RIP due
to the independency between left/right design vectors. But a symmetric
rank-one matrix sensing operator is not CI-RIP. In fact it is a biased
estimator since $E(\boldsymbol{x}{}^{\top}M\boldsymbol{x})=\mathrm{tr}(M)$
. To this end, we propose a shifted version of CI-RIP for symmetric
rank-one matrix sensing operator in the following theorem. This theorem
is the key tool in our analysis. 
\begin{thm}[Shifted CI-RIP]
 \label{thm:shifted-CI-RIP} Suppose $\boldsymbol{x}_{i}$ are independent
standard random Gaussian vectors, $M$ is a fixed symmetric rank $k$
matrix independent of $\boldsymbol{x}_{i}$ and $\boldsymbol{w}$
is a fixed vector. Then with a probability at least $1-\eta$, provided
$n\geq Ck^{3}d/\delta^{2}$ , 
\[
\|\frac{1}{2n}\mathcal{A}'\mathcal{A}(M)-\frac{1}{2}\mathrm{tr}(M)I-M\|_{2}\leq\delta\|M\|_{2}\ .
\]

\end{thm}
Theorem \ref{thm:shifted-CI-RIP} shows that $\frac{1}{2n}\mathcal{A}'\mathcal{A}(M)$
is nearly isometric after shifting by its expectation $\frac{1}{2}\mathrm{tr}(M)I$.
The RIP constant $\delta=O(\sqrt{k^{3}d/n})$ . In gFM, we choose
$M=M^{*}-M^{(t)}$ therefore $M$ is of rank $3k$ .

Under the same settings of Theorem \ref{thm:shifted-CI-RIP}, suppose
that $d\geq C$ then the following lemmas hold true with a probability
at least $1-\eta$ for fixed $\boldsymbol{w}$ and $M$ .
\begin{lem}
\label{lem:concentration-of-tr(M)} $|\frac{1}{n}\boldsymbol{1}{}^{\top}\mathcal{A}(M))-\mathrm{tr}(M)|\leq\delta\|M\|_{2}$
provided $n\geq Ck/\delta^{2}$ .
\end{lem}
{}
\begin{lem}
\label{lem:concentration-of-mu-Xw} $|\frac{1}{n}\boldsymbol{1}{}^{\top}X{}^{\top}\boldsymbol{w}|\leq\|\boldsymbol{w}\|_{2}\delta$
provided $n\geq C/\delta^{2}$ . 
\end{lem}
{}
\begin{lem}
\label{lem:concentration-ATXTw} $\|\frac{1}{n}\mathcal{A}'(X{}^{\top}\boldsymbol{w})\|_{2}\leq\|\boldsymbol{w}\|_{2}\delta$
provided $n\geq Cd/\delta^{2}$ .
\end{lem}
{}
\begin{lem}
\label{lem:concentation-XAM} $\|\frac{1}{n}X{}^{\top}\mathcal{A}(M)\|_{2}\leq\|M\|_{2}\delta$
provided $n\geq Ck^{2}d/\delta^{2}$ .
\end{lem}
{}
\begin{lem}
\label{lem:concentation-I-XXT} $\|I-\frac{1}{n}XX{}^{\top}\|_{2}\leq\delta$
provided $n\geq Cd/\delta^{2}$ .
\end{lem}
Equipping with the above lemmas, we construct our estimation sequence
as following.
\begin{lem}
	\label{lem:estimation-sqeuence-bound} 
	Let $M^{(t)},H_{1}^{(t)},h_{2}^{(t)},\boldsymbol{h}_{3}^{(t)}$
be defined as in Algorithm \ref{alg:OPFMs}. Define $\epsilon_{t}=\|\boldsymbol{w}^{*}-\boldsymbol{w}^{(t)}\|_{2} + \|M^{*}-M^{(t)}\|_{2} $
. Then with a probability at least $1-\eta$, provided $n\geq Ck^{3}d/\delta^{2}$ ,
\begin{align*}
H_{1}^{(t)}= & M^{*}-M^{(t)}+\mathrm{tr}(M^{*}-M^{(t)}) I +O(\delta\epsilon_{t})\ ,\ h_{2}^{(t)}=\mathrm{tr}(M^{*}-M^{(t)})+O(\delta\epsilon_{t})\\
	\boldsymbol{h}_{3}^{(t)}= & \boldsymbol{w}^{*}-\boldsymbol{w}^{(t)}+O(\delta\epsilon_{t})\ .
\end{align*}

\end{lem}
Suppose by construction, $\epsilon_{t}\rightarrow0$ when $t\rightarrow\infty$.
Then $H_{1}^{(t)}-h_{2}^{(t)}I+M^{(t)}\rightarrow M^{*}$ and $\boldsymbol{h}_{3}^{(t)}+\boldsymbol{w}^{(t)}\rightarrow\boldsymbol{w}^{*}$
and then the proof of Theorem \ref{thm:convergence-rate-of-OPFM}
is completed. In the following we only need to show that Lemma \ref{lem:estimation-sqeuence-bound}
constructs an estimation sequence with $\epsilon_{t}=O(\delta^{t})\rightarrow0$.
To this end, we need a few things from matrix perturbation theory.

By Theorem \ref{thm:convergence-rate-of-OPFM}, $U^{(t)}$ will converge
to $U^{*}$ up to column order perturbation. We use the largest canonical
angle to measure the subspace distance spanned by $U^{(t)}$ and $U^{*}$,
which is denoted as $\theta_{t}=\theta(U^{(t)},U^{*})$. For any matrix
$U$, it is well known \citep{zhu_angles_2013} that
\[
\sin\theta(U,U^{*})=\|U_{\perp}^{*}{}^{\top}U\|_{2},\ \cos\theta(U,U^{*})=\sigma_{k}\{U^{*}{}^{\top}U\},\ \tan\theta(U,U^{*})=\|U_{\perp}^{*}{}^{\top}U(U^{*}{}^{\top}U)^{-1}\|_{2}\ .
\]
 The last tangent equality allows us to bound the canonical angle after
QR decomposition. Suppose $U^{(t)}R=\hat{U}^{(t)}$ in the QR step
of Algorithm \ref{alg:OPFMs}, we have
\begin{align*}
\tan\theta(\hat{U}^{(t)},U^{*}) & =\|U_{\perp}^{*}{}^{\top}\hat{U}^{(t)}(U^{*}{}^{\top}\hat{U}^{(t)})^{-1}\|_{2}=\|U_{\perp}^{*}{}^{\top}U^{(t)}R(U^{*}{}^{\top}U^{(t)}R)^{-1}\|_{2}\\
 & =\|U_{\perp}^{*}{}^{\top}U^{(t)}(U^{*}{}^{\top}U^{(t)})^{-1}\|_{2}=\tan\theta(U^{(t)},U^{*})\ .
\end{align*}
 Therefore, it is more convenient to measure the subspace distance
by tangent function.

To show $\epsilon_{t}\rightarrow0$, we recursively define the following
variables:
\begin{align*}
\alpha_{t}\triangleq\tan\theta_{t},\ \beta_{t}\triangleq\|\boldsymbol{w}^{*}-\boldsymbol{w}^{(t)}\|_{2},\ \gamma_{t}\triangleq\|M^{*}-M^{(t)}\|_{2},\ \epsilon_{t}\triangleq\beta_{t}+\gamma_{t}\ .
\end{align*}
 The following lemma derives the recursive inequalities regarding
to $\{\alpha_{t},\beta_{t},\gamma_{t}\}$ .
\begin{lem}
\label{lem:abc-recursive} Under the same settings of Theorem \ref{thm:convergence-rate-of-OPFM},
suppose $\alpha_{t}\leq2$, $\delta\epsilon_{t}\leq4\sqrt{5}\sigma_{k}^{*}$,
then 
\begin{align*}
\alpha_{t+1} & \leq4\sqrt{5}\delta\sigma_{k}^{*-1}(\beta_{t}+\gamma_{t}),\ \beta_{t+1}\leq\delta(\beta_{t}+\gamma_{t}),\ \gamma_{t+1}\leq\alpha_{t+1}\|M^{*}\|_{2}+2\delta(\beta_{t}+\gamma_{t})\ .
\end{align*}

\end{lem}
In Lemma \ref{lem:abc-recursive}, when we choose $n$ such that $\delta=O(1/\sqrt{n})$
is small enough, $\{\alpha_{t},\beta_{t},\gamma_{t}\}$ will converge
to zero. The only question is the initial value $\{\alpha_{0},\beta_{0},\gamma_{0}\}$.
According to the initialization step of gFM, $\beta_{0}\leq\|\boldsymbol{w}^{*}\|_{2}$
and $\gamma_{0}\leq\|M^{*}\|_{2}$ . To bound $\alpha_{0}$ , we need
the following lemma which directly follows Wely's and Wedin's theorems
\citep{stewart_matrix_1990}.
\begin{lem}
\label{lem:init-singular-vector-perturbation} Denote $U$ and $\widetilde{U}$
as the top-$k$ left singular vectors of $M$ and $\widetilde{M}=M+O(\epsilon)$
respectively. The $i$-th singular value of $M$ is $\sigma_{i}$.
Suppose that $\epsilon\leq\frac{\sigma_{k}-\sigma_{k+1}}{4}$. Then
the largest canonical angle between $U$ and $\widetilde{U}$, denoted
as $\theta(U,\widetilde{U})$, is bounded by $\sin\theta(U,\widetilde{U})\leq2\epsilon/(\sigma_{k}-\sigma_{k+1})$
.
\end{lem}
According to Lemma \ref{lem:init-singular-vector-perturbation}, when
$2\delta(\|\boldsymbol{w}^{*}\|_{2}+\|M^{*}\|_{_{2}})\leq\sigma_{k}^{*}/4$,
we have $\sin\theta_{0}\leq4\delta(\|\boldsymbol{w}^{*}\|_{2}+\|M^{*}\|_{_{2}})/\sigma_{k}^{*}$.
Therefore, $\alpha_{0}\leq2$ provided $\delta\leq\sigma_{k}^{*}/[8(\|\boldsymbol{w}^{*}\|_{2}+\|M^{*}\|_{_{2}})]$
.
\begin{proof}[Proof of Theorem \ref{thm:convergence-rate-of-OPFM}]

Suppose that at step $t$, $\alpha_{t}\leq2$, $\delta\epsilon_{t}\leq4\sqrt{5}\sigma_{k}^{*}$,
from Lemma \ref{lem:abc-recursive},
\begin{align*}
\beta_{t+1}+\gamma_{t+1}\leq & \beta_{t+1}+\alpha_{t+1}\|M^{*}\|_{2}+2\delta(\beta_{t}+\gamma_{t})\leq\delta\epsilon_{t}+4\sqrt{5}\delta\sigma_{k}^{*-1}\epsilon_{t}\|M^{*}\|_{2}+2\delta\epsilon_{t}\\
= & (4\sqrt{5}\sigma_{1}^{*}/\sigma_{k}^{*}+3)\delta\epsilon_{t}\ .
\end{align*}
 Therefore,
\begin{align*}
\epsilon_{t} & =\beta_{t}+\gamma_{t}\leq[(4\sqrt{5}\sigma_{1}^{*}/\sigma_{k}^{*}+3)\delta]^{t}(\beta_{0}+\gamma_{0})\\
\alpha_{t+1} & \leq4\sqrt{5}\delta\sigma_{k}^{*-1}(\beta_{t}+\gamma_{t})\leq4\sqrt{5}\delta\sigma_{k}^{*-1}[(4\sqrt{5}\sigma_{1}^{*}/\sigma_{k}^{*}+3)\delta]^{t}(\beta_{0}+\gamma_{0})\ .
\end{align*}
 Clearly we need $(4\sqrt{5}\sigma_{1}^{*}/\sigma_{k}^{*}+3)\delta<1$
to ensure convergence, which is guaranteed by $\delta<\frac{\sigma_{k}^{*}}{4\sqrt{5}\sigma_{1}^{*}+3\sigma_{k}^{*}}  $ .
 To ensure the recursive inequality holds for any $t$, we require
$\alpha_{t+1}\leq2$, which is guaranteed by 
	\[ 4\sqrt{5}(\beta_{0}+\gamma_{0})\delta/\sigma_{k}^{*}\leq2\Leftrightarrow\delta\leq\frac{\sigma_{k}^{*}}{2\sqrt{5}(\sigma_{1}^{*}+\beta_{0})} \ .\]
	To ensure the condition $\delta\epsilon_{t}\leq4\sqrt{5}\sigma_{k}^{*}$,
\[
\delta\leq4\sqrt{5}\sigma_{k}^{*}/\epsilon_{0}=4\sqrt{5}\sigma_{k}^{*}/(\sigma_{1}^{*}+\beta_{0})\Rightarrow\delta\leq4\sqrt{5}\sigma_{k}^{*}/\epsilon_{t}\ .
\]

In summary, when 
\begin{align*}
 & \delta\leq\min\left\{ \frac{\sigma_{k}^{*}}{4\sqrt{5}(\sigma_{1}^{*}+\beta_{0})},\frac{\sigma_{k}^{*}}{4\sqrt{5}\sigma_{1}^{*}+3\sigma_{k}^{*}},\frac{\sigma_{k}^{*}}{2\sqrt{5}(\sigma_{1}^{*}+\beta_{0})},\frac{\sigma_{k}^{*}}{8(\sigma_{1}^{*}+\beta_{0})}\right\} \\
\Leftarrow & \delta\leq\frac{\sigma_{k}^{*}}{4\sqrt{5}\sigma_{1}^{*}+3\sigma_{k}^{*}+4\sqrt{5}\beta_{0}}\ .
\end{align*}
 we have 
\[
\epsilon_{t}=[(4\sqrt{5}\sigma_{1}^{*}/\sigma_{k}^{*}+3)\delta]^{t}(\sigma_{1}^{*}+\gamma_{0})\ .
\]
 To simplify the result, replace $\delta$ with $\delta_{1}=(4\sqrt{5}\sigma_{1}^{*}/\sigma_{k}^{*}+3)\delta$.
The proof is completed.
\end{proof}

\subsection{Noisy Case \label{sub:Noisy-Case}}

In this subsection, we analyze the performance of gFM under noisy
setting. Suppose that $M^{*}$ is no longer low rank, $M^{*}=U^{*}\Lambda^{*}U^{*}{}^{\top}+U_{\perp}^{*}\Lambda_{\perp}^{*}U_{\perp}^{*}{}^{\top}$
where $\Lambda_{\perp}^{*}=\mathrm{diag}(\lambda_{k+1},\cdots,\lambda_{d})$
is the residual spectrum. Denote $M_{k}^{*}=U^{*}\Lambda^{*}U^{*}{}^{\top}$
to be the best rank $k$ approximation of $M^{*}$ and $M_{\perp}^{*}=M^{*}-M^{*}_k$.
The additive noise $\xi_{i}$'s are independently sampled from subgaussian
with proxy variance $\xi$.

First we generalize the above theorems and lemmas to noisy case.
\begin{lem}
\label{lem:noisy-concentration-sensing-operator} Suppose that in
Eq. (\ref{eq:y_i=00003Dx_iw*+x_iM*x_i+e_i}) $\boldsymbol{x}_{i}$'s
are independent standard random Gaussian vectors. $M$ is a fixed rank $k$ matrix.
$M_{\perp}^{*}\not=\boldsymbol{0}$ and $\xi>0$. Then provided $n\geq Ck^{3}d/\delta^{2}$,
with a probability at least $1-\eta$, 
\begin{align}
 & \|\frac{1}{2n}\mathcal{A}'\mathcal{A}(M^{*}-M)-\frac{1}{2}\mathrm{tr}(M_{k}^{*}-M)I-(M_{k}^{*}-M)\|_{2}\leq\delta\|M_{k}^{*}-M\|_{2}+C\sigma_{k+1}^{*}d^{2}/\sqrt{n}\label{eq:noisy-AAM-concentration}\\
 & |\frac{1}{n}\boldsymbol{1}{}^{\top}\mathcal{A}(M^{*}-M)-\mathrm{tr}(M_{k}^{*}-M)|\leq\delta\|M_{k}^{*}-M\|_{2}+C\sigma_{k+1}^{*}d^{2}/\sqrt{n}\label{eq:noisy-1AM-concentration}\\
 & \|\frac{1}{n}X{}^{\top}\mathcal{A}(M^{*}-M)\|_{2}\leq\delta\|M_{k}^{*}-M\|_{2}+C\sigma_{k+1}^{*}d^{2}/\sqrt{n}\label{eq:noisy-XAM-concentration}\\
 & \|\frac{1}{n}\mathcal{A}'(X{}^{\top}\boldsymbol{w})\|_{2}\leq\delta\|\boldsymbol{w}\|_{2},\ \|\frac{1}{n}\boldsymbol{1}{}^{\top}X{}^{\top}\boldsymbol{w}\|_{2}\leq\delta\|\boldsymbol{w}\|_{2}\ .\label{eq:noisy-concentration-AXw-and-1xw}
\end{align}

\end{lem}
Define $\gamma_{t}=\|M_{k}^{*}-M^{(t)}\|_{2}$ similar to the noise-free case.
According to Lemma \ref{lem:noisy-concentration-sensing-operator},
when $\xi=0$, for $n\geq Ck^{3}d/\delta^{2}$, 
\begin{align*}
H_{1}^{(t)}= & M_{k}^{*}-M^{(t)}+\frac{1}{2}\mathrm{tr}(M_{k}^{*}-M^{(t)})I+O(\delta\epsilon_{t}+C\sigma_{k+1}^{*}d^{2}/\sqrt{n})\\
h_{2}^{(t)}= & \mathrm{tr}(M^{*}-M^{(t)})+O(\delta\epsilon_{t}+C\sigma_{k+1}^{*}d^{2}/\sqrt{n})\\
	\boldsymbol{h}_{3}^{(t)}= & \boldsymbol{w}^{*}-\boldsymbol{w}^{(t)}+O(\delta\epsilon_{t}+C\sigma_{k+1}^{*}d^{2}/\sqrt{n})\ .
\end{align*}
 Define $r=C\sigma_{k+1}^{*}d^{2}/\sqrt{n}$. If $\xi>0$, it is easy
to check that the perturbation becomes $\hat{r}=r+O(\xi/\sqrt{n})$
. Therefore we uniformly use $r$ to present the perturbation term.
The recursive inequalities regarding to the recovery error is constructed
in Lemma \ref{lem:noisy-recursive-inequality}.
\begin{lem}
\label{lem:noisy-recursive-inequality} Under the same settings of
Lemma \ref{lem:noisy-concentration-sensing-operator}, define $\rho\triangleq2\sigma_{k+1}^{*}/(\sigma_{k}^{*}+\sigma_{k+1}^{*})$.
Suppose that at any step $i$, $0\leq i\leq t$ , $\alpha_{i}\leq2$
. When provided $4\sqrt{5}(\delta\epsilon_{t}+r)\leq\sigma_{k}^{*}-\sigma_{k+1}^{*}$,
\begin{align*}
\alpha_{t+1}\leq & \rho\alpha_{t}+\frac{4\sqrt{5}}{\sigma_{k}^{*}+\sigma_{k+1}^{*}}\delta\epsilon_{t}+\frac{4\sqrt{5}}{\sigma_{k}^{*}+\sigma_{k+1}^{*}}r\ ,\ \beta_{t+1}\leq\delta\epsilon_{t}+r\:,\ \gamma_{t+1}\leq\alpha_{t+1}\|M^{*}\|_{2}+2\delta\epsilon_{t}+2r\ .
\end{align*}

\end{lem}
The solution to the recursive inequalities in Lemma \ref{lem:noisy-recursive-inequality}
is non-trivial. Comparing to the inequalities in Lemma \ref{lem:abc-recursive},
$\alpha_{t+1}$ is bounded by $\alpha_{t}$ in noisy case. Therefore,
if we simply follow Lemma \ref{lem:abc-recursive} to construct recursive
inequality about $\epsilon_{t}$ , we will quickly be overloaded by
recursive expansion terms. The key construction of our solution is
to bound the term $\alpha_{t}+8\sqrt{5}/(\sigma_{k}^{*}+\sigma_{k+1}^{*})\delta\epsilon_{t}$
. The solution is given in the following theorem.
\begin{thm}
\label{thm:the-solution-of-noisy-recursive-inequality} Define constants
\begin{align*}
c= & 4\sqrt{5}/(\sigma_{k}^{*}+\sigma_{k+1}^{*})\ ,\ b=3+4\sqrt{5}\sigma_{1}^{*}/(\sigma_{k}^{*}+\sigma_{k+1}^{*})\ ,\ q=(1+\rho)/2\ .
\end{align*}
 Then for any $t\geq0$,
\begin{align}
\alpha_{t}+2c\delta\epsilon_{t}\leq & q^{t}\left(2-\frac{(1+\rho)cr}{1-q}\right)+\frac{(1+\rho)cr}{1-q}\ .\label{eq:the-solution-of-noisy-ri}
\end{align}
 provided
\begin{gather}
\delta\leq\min\{\frac{1-\rho}{4\rho\sigma_{1}^{*}c},\frac{\rho}{2b}\}\ ,\ (2+c(\sigma_{k}^{*}-\sigma_{k+1}^{*}))\delta\epsilon_{0}+r\leq(\sigma_{k}^{*}-\sigma_{k+1}^{*})\label{eq:the-solution-of-noisy-ri-constraint-on-delta}\\
4\sqrt{5}\left(4+2c(\sigma_{k}^{*}-\sigma_{k+1}^{*})\right)\delta\epsilon_{0}+4\sqrt{5}\left(4+(\sigma_{k}^{*}-\sigma_{k+1}^{*})\right)r\leq(\sigma_{k}^{*}-\sigma_{k+1}^{*})^{2}\ .\nonumber 
\end{gather}

\end{thm}
Theorem \ref{thm:the-solution-of-noisy-recursive-inequality} gives
the convergence rate of gFM under noisy settings. We bound $\alpha_{t}+2c\delta\epsilon_{t}$
as the index of recovery error, whose convergence rate is linear.
The convergence rate is controlled by $q$, a constant depends on
the eigen gap $\sigma_{k+1}^{*}/\sigma_{k}^{*}$ . The final  recovery
error is bounded by $O(r/(1-q))$ . Eq. (\ref{eq:the-solution-of-noisy-ri-constraint-on-delta})
is the small noise condition to ensure the noisy recovery is possible.
Generally speaking, learning a $d \times d$ matrix with $O(d)$ samples
is an ill-conditioned problem when the target matrix is full rank.
The small noise condition given by Eq. (\ref{eq:the-solution-of-noisy-ri-constraint-on-delta})
essentially says that $M^{*}$ can be slightly deviated from low rank
manifold and the noise shouldn't be too large to blur the spectrum
of $M^{*}$. When the noise is large, Eq. (\ref{eq:the-solution-of-noisy-ri-constraint-on-delta})
will be satisfied with $n=O(d^{2})$ which is the information-theoretical
lower bound for recovering a full rank matrix.

\section{Conclusion}

In this paper, we propose a provable efficient algorithm to solve
generalized Factorization Machine (gFM) and rank-one matrix sensing. Our method is based
on an one-pass alternating updating framework. The proposed algorithm
is able to learn gFM within $O(kd)$ memory on steaming data, has
linear convergence rate and only requires matrix-vector product implementation.
The algorithm takes no more than $O(k^{3}d\log\left(1/\epsilon\right))$
instances to achieve $O(\epsilon)$ recovery error.

\small
\bibliographystyle{plainnat}
\bibliography{refs}

\begin{thebibliography}{25}
\providecommand{\natexlab}[1]{#1}
\providecommand{\url}[1]{\texttt{#1}}
\expandafter\ifx\csname urlstyle\endcsname\relax
  \providecommand{\doi}[1]{doi: #1}\else
  \providecommand{\doi}{doi: \begingroup \urlstyle{rm}\Url}\fi

\bibitem[Blondel et~al.(2016)Blondel, Ishihata, Fujino, and
  Ueda]{blondel_polynomial_2016}
Mathieu Blondel, Masakazu Ishihata, Akinori Fujino, and Naonori Ueda.
\newblock Polynomial {{Networks}} and {{Factorization Machines}}: {{New
  Insights}} and {{Efficient Training Algorithms}}.
\newblock pages 850--858, 2016.

\bibitem[Cai and Zhang(2015)]{cai_rop_2015}
T.~Tony Cai and Anru Zhang.
\newblock {{ROP}}: {{Matrix}} recovery via rank-one projections.
\newblock \emph{The Annals of Statistics}, 43\penalty0 (1):\penalty0 102--138,
  2015.

\bibitem[Cand{\`e}s and Recht(2009)]{candes_exact_2009}
Emmanuel~J Cand{\`e}s and Benjamin Recht.
\newblock Exact matrix completion via convex optimization.
\newblock \emph{Foundations of Computational mathematics}, 9\penalty0
  (6):\penalty0 717--772, 2009.

\bibitem[Candes et~al.(2011)Candes, Eldar, Strohmer, and
  Voroninski]{candes_phase_2011}
Emmanuel~J. Candes, Yonina Eldar, Thomas Strohmer, and Vlad Voroninski.
\newblock Phase {{Retrieval}} via {{Matrix Completion}}.
\newblock \emph{arXiv:1109.0573}, 2011.

\bibitem[Chen et~al.(2015)Chen, Chi, and Goldsmith]{chen_exact_2015}
Yuxin Chen, Yuejie Chi, and Andrea~J. Goldsmith.
\newblock Exact and stable covariance estimation from quadratic sampling via
  convex programming.
\newblock \emph{Information Theory, IEEE Transactions on}, 61\penalty0
  (7):\penalty0 4034--4059, 2015.

\bibitem[Davenport and Romberg(2016)]{davenport_overview_2016}
Mark~A. Davenport and Justin Romberg.
\newblock An overview of low-rank matrix recovery from incomplete observations.
\newblock \emph{arXiv:1601.06422}, 2016.

\bibitem[Hardt(2013)]{moritz_hardt_understanding_2013}
Moritz Hardt.
\newblock Understanding {{Alternating Minimization}} for {{Matrix Completion}}.
\newblock \emph{arXiv:1312.0925}, 2013.

\bibitem[Hardt and Price(2013)]{hardt_noisy_2013}
Moritz Hardt and Eric Price.
\newblock The {{Noisy Power Method}}: {{A Meta Algorithm}} with
  {{Applications}}.
\newblock \emph{arXiv:1311.2495}, 2013.

\bibitem[Hardt and Wootters(2014)]{hardt_fast_2014}
Moritz Hardt and Mary Wootters.
\newblock Fast matrix completion without the condition number.
\newblock \emph{arXiv:1407.4070}, 2014.

\bibitem[Hong et~al.(2013)Hong, Doumith, and
  Davison]{hong_co-factorization_2013}
Liangjie Hong, Aziz~S. Doumith, and Brian~D. Davison.
\newblock Co-factorization {{Machines}}: {{Modeling User Interests}} and
  {{Predicting Individual Decisions}} in {{Twitter}}.
\newblock In \emph{WSDM}, pages 557--566, 2013.

\bibitem[Jain and Dhillon(2013)]{jain_provable_2013}
Prateek Jain and Inderjit~S. Dhillon.
\newblock Provable inductive matrix completion.
\newblock \emph{arXiv:1306.0626}, 2013.

\bibitem[Jain et~al.(2012)Jain, Netrapalli, and Sanghavi]{jain_low-rank_2012}
Prateek Jain, Praneeth Netrapalli, and Sujay Sanghavi.
\newblock Low-rank {{Matrix Completion}} using {{Alternating Minimization}}.
\newblock \emph{arXiv:1212.0467}, 2012.

\bibitem[Jain et~al.(2013)Jain, Netrapalli, and Sanghavi]{jain_low-rank_2013}
Prateek Jain, Praneeth Netrapalli, and Sujay Sanghavi.
\newblock Low-rank {{Matrix Completion Using Alternating Minimization}}.
\newblock In \emph{STOC}, pages 665--674, 2013.

\bibitem[Koltchinskii(2011)]{koltchinskii_oracle_2011}
V.~Koltchinskii.
\newblock \emph{Oracle {{Inequalities}} in {{Empirical Risk Minimization}} and
  {{Sparse Recovery Problems}}}, volume 2033.
\newblock Springer, 2011.

\bibitem[Kueng et~al.(2014)Kueng, Rauhut, and Terstiege]{kueng_low_2014}
Richard Kueng, Holger Rauhut, and Ulrich Terstiege.
\newblock Low rank matrix recovery from rank one measurements.
\newblock \emph{arXiv:1410.6913}, 2014.

\bibitem[Lee et~al.(2013)Lee, Wu, and Bresler]{lee_near_2013}
Kiryung Lee, Yihong Wu, and Yoram Bresler.
\newblock Near {{Optimal Compressed Sensing}} of {{Sparse Rank-One Matrices}}
  via {{Sparse Power Factorization}}.
\newblock \emph{arXiv:1312.0525}, 2013.

\bibitem[Netrapalli et~al.(2013)Netrapalli, Jain, and
  Sanghavi]{netrapalli_phase_2013}
Praneeth Netrapalli, Prateek Jain, and Sujay Sanghavi.
\newblock Phase {{Retrieval}} using {{Alternating Minimization}}.
\newblock \emph{arXiv:1306.0160}, 2013.

\bibitem[Rendle(2010)]{rendle_factorization_2010}
Steffen Rendle.
\newblock Factorization machines.
\newblock In \emph{ICDM}, pages 995--1000, 2010.

\bibitem[Rendle et~al.(2011)Rendle, Gantner, Freudenthaler, and
  Schmidt-Thieme]{rendle_fast_2011}
Steffen Rendle, Zeno Gantner, Christoph Freudenthaler, and Lars Schmidt-Thieme.
\newblock Fast {{Context}}-aware {{Recommendations}} with {{Factorization
  Machines}}.
\newblock In \emph{SIGIR}, pages 635--644, 2011.

\bibitem[Stewart and Sun(1990)]{stewart_matrix_1990}
G.~W. Stewart and Ji-guang Sun.
\newblock \emph{Matrix {{Perturbation Theory}}}.
\newblock Academic Press, 1990.

\bibitem[Tropp(2014)]{tropp_convex_2014}
Joel~A. Tropp.
\newblock Convex recovery of a structured signal from independent random linear
  measurements.
\newblock \emph{arXiv:1405.1102}, 2014.

\bibitem[Zhao et~al.(2015{\natexlab{a}})Zhao, Wang, and
  Liu]{zhao_nonconvex_2015}
Tuo Zhao, Zhaoran Wang, and Han Liu.
\newblock Nonconvex {{Low Rank Matrix Factorization}} via {{Inexact First Order
  Oracle}}.
\newblock 2015{\natexlab{a}}.

\bibitem[Zhao et~al.(2015{\natexlab{b}})Zhao, Wang, and
  Liu]{zhao_nonconvex_2015-1}
Tuo Zhao, Zhaoran Wang, and Han Liu.
\newblock A {{Nonconvex Optimization Framework}} for {{Low Rank Matrix
  Estimation}}.
\newblock In \emph{NIPS}, pages 559--567, 2015{\natexlab{b}}.

\bibitem[Zhong et~al.(2015)Zhong, Jain, and Dhillon]{zhong_efficient_2015}
Kai Zhong, Prateek Jain, and Inderjit~S. Dhillon.
\newblock Efficient matrix sensing using rank-1 gaussian measurements.
\newblock In \emph{Algorithmic {{Learning Theory}}}, pages 3--18, 2015.

\bibitem[Zhu and Knyazev(2013)]{zhu_angles_2013}
Peizhen Zhu and Andrew~V. Knyazev.
\newblock Angles between subspaces and their tangents.
\newblock \emph{Journal of Numerical Mathematics}, 21\penalty0 (4), 2013.

\end{thebibliography}

\appendix

\section{Preliminary\label{sec:appendix-preliminary}}

In this section, we present several important theorems and lemmas
in our analysis.

The following concentration inequalities are well known.
\begin{lem}
\label{lem:subgaussian-concentration} Let $x_{i}$ be zero-mean sub-Gaussian
distribution with variance proxy $\sigma^{2}$. Denote $S_{n}=\sum_{i=1}^{n}a_{i}x_{i}$
for a fixed sequence $\{a_{i}\}$. Then
\[
\mathrm{Pr}(|S_{n}|>t)\leq2\exp(-\frac{t^{2}}{2\sigma^{2}(\sum_{i=1}^{n}a_{i}^{2})})\ .
\]
 That is, with a probability at least $1-\eta$, 
\[
|S_{n}|\leq\sigma\sqrt{\sum_{i=1}^{n}a_{i}^{2}}\sqrt{2\log(2/\eta)}\ .
\]
\end{lem}
\begin{cor}
\label{cor:gaussian-distribution-var-concentraion} Let $x_{i}\sim\mathcal{N}(0,1)$
be a standard Gaussian distribution. Then with a probability at least
$1-\eta$,
\begin{align*}
\sum_{i=1}^{n}a_{i}(x_{i}^{2}-1)\leq & 2\sqrt{\sum_{i=1}^{n}a_{i}}\sqrt{2\log(2/\eta)}\ .
\end{align*}

\end{cor}
For random matrix, we have matrix concentration inequalities \citep{tropp_introduction_2015}.
\begin{thm}[Matrix Bernstein's Inequality \citep{tropp_introduction_2015}]
 \label{thm:matrix-bernstein} Suppose $\{S_{i}\}_{i=1}^{n}$ are
set of independent random matrices of dimension $d_{1}\times d_{2}$,
\[
\|S_{i}-ES_{i}\|\leq L\ .
\]
 Define
\[
Z=\sum_{i=1}^{n}S_{i},\ \sigma^{2}=\frac{1}{n}\max(E\|(Z-EZ)(Z-EZ){}^{\top}\|_{2},E\|(Z-EZ){}^{\top}(Z-EZ)\|_{2})\ .
\]
 The with a probability at least $1-\delta$, for any $0<\epsilon<1$,
\[
\frac{1}{n}\|Z-EZ\|_{2}\leq9\epsilon\sqrt{\log((d_{1}+d_{2})/\delta)}
\]
 provided
\[
n\geq\max(\sigma^{2},L)/\epsilon^{2}\ .
\]
 And for any $n$,
\begin{align*}
\frac{1}{n}\|Z-EZ\|_{2}\leq & \frac{4}{3}\frac{L}{n}\log((d_{1}+d_{2})/\delta)+3\sqrt{2\frac{\sigma^{2}}{n}\log((d_{1}+d_{2})/\delta)}\ .
\end{align*}

\end{thm}
Using matrix Bernstein's inequality, we can bound the covariance estimator.
\begin{cor}[Matrix Bernstein's Inequality for Covariance Estimator \citep{tropp_introduction_2015}]
 \label{cor:matrix-covariance-concentration} Suppose $\boldsymbol{x}_{i}\in\mathbb{R}^{d},i=1,2,\cdots,n$
are independent random variables with zero mean.
\[
\|\boldsymbol{x}_{i}\|^{2}\leq B,\ A=E(\boldsymbol{x}_{i}\boldsymbol{x}_{i}{}^{\top})
\]
 Then with a probability at least $1-\delta$,
\[
\|\frac{1}{n}\sum_{i=1}^{n}\boldsymbol{x}_{i}\boldsymbol{x}{}_{i}^{\top}-A\|_{2}\leq9\epsilon\sqrt{\log(2d/\delta)/n}
\]
 provided 
\[
n\geq\max(B\|A\|,B)/\epsilon^{2}\ .
\]

\end{cor}

\section{Proof of Lemmas}

\subsection{Proof of Lemma \ref{lem:concentration-of-tr(M)}}
\begin{proof}
Denote the eigenvalue decomposition of $M$ as 
\[
M=U\Lambda U{}^{\top}=U\mathrm{diag}(\lambda_{1},\lambda_{2},\cdots,\lambda_{k})U{}^{\top}
\]
 Since Gaussian distribution is rotation invariant, $\hat{\boldsymbol{x}}_{i}=U{}^{\top}\boldsymbol{x}_{i}$
also follows standard Gaussian distribution. 
\begin{align*}
\boldsymbol{x}_{i}{}^{\top}M\boldsymbol{x}_{i}= & \boldsymbol{x}_{i}{}^{\top}U\Lambda U{}^{\top}\boldsymbol{x}_{i}=|\hat{\boldsymbol{x}}_{i}{}^{\top}\Lambda\hat{\boldsymbol{x}}_{i}|=\sum_{j=1}^{k}\lambda_{j}\hat{\boldsymbol{x}}_{i,j}^{2}\ .
\end{align*}
 It is easy to see that $E(\boldsymbol{x}_{i}{}^{\top}M\boldsymbol{x}_{i})=\sum_{j}\lambda_{j}=\mathrm{tr}(M)$.
Define
\begin{align*}
a_{i}\triangleq & \boldsymbol{x}_{i}{}^{\top}M\boldsymbol{x}_{i}-\mathrm{tr}(M)=\sum_{j=1}^{d}\lambda_{j}(\hat{\boldsymbol{x}}_{i,j}^{2}-1)
\end{align*}
 According to Corollary \ref{cor:gaussian-distribution-var-concentraion},
for a fixed $i,$ with a probability at least $1-\eta$, 
\[
|a_{i}|\leq2\|M\|_{F}\sqrt{2\log(2/\eta)}\ .
\]
 Then for any $i$, with a probability at least $1-\eta$, 
\[
|a_{i}|\leq2\|M\|_{F}\sqrt{2\log(2n/\eta)}\ .
\]
 Apply Corollary \ref{cor:gaussian-distribution-var-concentraion}
again, with a probability at least $1-2\eta$,
\begin{align*}
|\frac{1}{n}\sum_{i=1}^{n}a_{i}-\mathrm{tr}(M)|\leq & 2\|M\|_{F}\sqrt{2\log(2n/\eta)}\sqrt{2\log(2/\eta)/n}\\
\leq & 2\sqrt{k}\|M\|_{2}\sqrt{2\log(2n/\eta)}\sqrt{2\log(2/\eta)/n}\ .
\end{align*}

Denote $\delta=2\sqrt{k}\sqrt{2\log(2n/\eta)}\sqrt{2\log(2/\eta)/n}$.
Then when $n\geq Ck/\delta^{2}$,
\[
|\frac{1}{n}\sum_{i=1}^{n}a_{i}-\mathrm{tr}(M)|\leq\|M\|_{2}\delta\ .
\]

\end{proof}

\subsection{Proof of Lemma \ref{lem:concentration-of-mu-Xw}}
\begin{proof}
Define random variable
\begin{align*}
a_{i}= & \boldsymbol{x}_{i}{}^{\top}\boldsymbol{w},\ Ea_{i}=0\\
Ea_{i}^{2}\leq & \|\boldsymbol{w}\|_{2}^{2}
\end{align*}
 Then according to Lemma \ref{lem:subgaussian-concentration}, with
a probability at least $1-\eta$,
\[
|\frac{1}{n}\sum_{i=1}^{n}a_{i}|\leq\|\boldsymbol{w}\|_{2}\sqrt{2\log(2/\eta)/n}\ .
\]

\end{proof}

\subsection{Proof of Lemma \ref{lem:concentation-XAM}}
\begin{proof}
Define random vector 
\[
\boldsymbol{a}_{i}=\boldsymbol{x}_{i}\boldsymbol{x}_{i}{}^{\top}M\boldsymbol{x}_{i},\ E\boldsymbol{a}_{i}=0\ .
\]
 With a probability at least $(1-\eta_{1})(1-\eta_{2})$, 
\begin{align*}
\|\boldsymbol{a}_{i}\|_{2}= & \|\boldsymbol{x}_{i}\boldsymbol{x}_{i}{}^{\top}M\boldsymbol{x}_{i}\|_{2}\leq\|\boldsymbol{x}_{i}{}^{\top}M\boldsymbol{x}_{i}\|_{2}\|\boldsymbol{x}_{i}\|_{2}\\
\leq & (|\mathrm{tr}(M)|+2\|M\|_{F}\sqrt{2\log(2n/\eta_{1})})\sqrt{2d\log(2n/\eta_{2})}\\
\triangleq & c_{1}\sqrt{2d\log(2n/\eta_{2})}\ .
\end{align*}
\begin{align*}
\|E\boldsymbol{a}_{i}{}^{\top}\boldsymbol{a}_{i}\|_{2}= & \|\boldsymbol{x}_{i}{}^{\top}M\boldsymbol{x}_{i}\boldsymbol{x}_{i}{}^{\top}\boldsymbol{x}_{i}\boldsymbol{x}_{i}{}^{\top}M\boldsymbol{x}_{i}\|_{2}\\
\leq & (\boldsymbol{x}_{i}{}^{\top}M\boldsymbol{x}_{i})^{2}\|\boldsymbol{x}_{i}\|_{2}^{2}\\
\leq & 2c_{1}^{2}d\log(2n/\eta_{2})\ .
\end{align*}

By matrix Bernstein's inequality, the concentration holds when
\begin{align*}
n\geq & \frac{1}{\epsilon^{2}}\max\{c_{1}\sqrt{2d\log(2n/\eta_{2})},2c_{1}^{2}d\log(2n/\eta_{2})\}\\
= & \frac{1}{\epsilon^{2}}O(k^{2}d\|M\|_{2}^{2})\ .
\end{align*}
 Therefore, after taking the union bound, there exists some constant
$C_{2}=O(\log(2n/\eta))$, 
\[
\|\frac{1}{n}\sum_{i=1}^{n}\boldsymbol{a_{i}}\|_{2}\leq\epsilon
\]
\[
n\geq C_{2}k^{2}d\|M\|_{2}^{2}\log(2(d+1)/\eta)/\epsilon^{2}\ .
\]
 Denote $\delta=\|M\|_{2}/\epsilon$. Then when $n\geq Ck^{2}d/\delta$,
\[
\|\frac{1}{n}\sum_{i=1}^{n}\boldsymbol{a_{i}}\|_{2}\leq\|M\|_{2}\delta\ .
\]

\end{proof}

\subsection{Proof of Lemma \ref{lem:concentration-ATXTw}}

\begin{align*}
\|\frac{1}{n}\mathcal{A}'(X{}^{\top}\boldsymbol{w})\|_{2}= & \|\frac{1}{n}\sum_{i=1}^{n}\boldsymbol{x}_{i}\boldsymbol{x}_{i}{}^{\top}\boldsymbol{w}\boldsymbol{x}_{i}{}^{\top}\|_{2}\ .
\end{align*}
\begin{align*}
E\{\boldsymbol{x}_{i}\boldsymbol{x}_{i}{}^{\top}\boldsymbol{w}\boldsymbol{x}_{i}{}^{\top}\}=0
\end{align*}
\begin{align*}
\|\boldsymbol{x}_{i}\boldsymbol{x}_{i}{}^{\top}\boldsymbol{w}\boldsymbol{x}_{i}{}^{\top}\|_{2}\leq & \|\boldsymbol{x}_{i}{}^{\top}\boldsymbol{w}\|_{2}\|\boldsymbol{x}_{i}\|_{2}^{2}\\
\leq & 2\|\boldsymbol{w}\|_{2}\sqrt{2\log(2/\eta)}(d+2\sqrt{2d\log(2n/\eta)})\\
\leq & 4\|\boldsymbol{w}\|_{2}\sqrt{2\log(2/\eta)}d
\end{align*}
 provided $d\geq8\log(2n/\eta)$. Now considering
\begin{align*}
\{E\boldsymbol{x}_{i}\boldsymbol{x}_{i}{}^{\top}\boldsymbol{w}\boldsymbol{x}_{i}{}^{\top}\boldsymbol{x}_{i}\boldsymbol{w}{}^{\top}\boldsymbol{x}_{i}\boldsymbol{x}_{i}{}^{\top}\}_{pq}= & E\{(\sum x_{p}x_{q}w_{i1}x_{i1}w_{i2}x_{i2}x_{i3}^{2})\}
\end{align*}
 When $p\not=q$,
\begin{align*}
E\{(\sum x_{p}x_{q}w_{i1}x_{i1}w_{i2}x_{i2}x_{i3}^{2})\} & =E\{(2\sum_{i3}x_{p}x_{q}w_{p}x_{p}w_{q}x_{q}x_{i3}^{2})\}\\
 & =E\{(2\sum_{i3}x_{p}^{2}x_{q}^{2}w_{p}w_{q}x_{i3}^{2})\}\\
 & =2w_{p}w_{q}E\{(\sum_{i3}x_{p}^{2}x_{q}^{2}x_{i3}^{2})\}\\
 & =2w_{p}w_{q}d
\end{align*}
 When $p=q$,
\begin{align*}
E\{(\sum x_{p}x_{q}w_{i1}x_{i1}w_{i2}x_{i2}x_{i3}^{2})\}= & E\{(\sum x_{p}^{2}w_{i1}x_{i1}w_{i2}x_{i2}x_{i3}^{2})\}\\
= & E\{(\sum x_{p}^{2}w_{p}x_{p}w_{p}x_{p}x_{i3}^{2}+\sum x_{p}^{2}w_{j}x_{j}w_{j}x_{j}x_{i3}^{2}+\sum x_{p}^{2}w_{i3}x_{i3}w_{i3}x_{i3}x_{i3}^{2})\}\\
= & E\{(\sum_{i3\not=p}x_{p}^{4}w_{p}^{2}x_{i3}^{2}+\sum_{j\not=i3\not=p}x_{p}^{2}w_{j}^{2}x_{j}^{2}x_{i3}^{2}+\sum_{i3\not=p}x_{p}^{2}w_{i3}^{2}x_{i3}^{4})\}\\
= & w_{p}^{2}(d-1)+\sum_{j\not=p}w_{j}^{2}(d-1)+\sum_{i3\not=p}w_{i3}^{2}\\
= & w_{p}^{2}(d-1)+\sum_{j\not=p}w_{j}^{2}d=w_{p}^{2}(d-1)+\sum_{j=1}^{d}w_{j}^{2}d-w_{p}^{2}d\\
= & \sum_{j=1}^{d}w_{j}^{2}d-w_{p}^{2}
\end{align*}
 Therefore,
\begin{align*}
E\boldsymbol{x}_{i}\boldsymbol{x}_{i}{}^{\top}\boldsymbol{w}\boldsymbol{x}_{i}{}^{\top}\boldsymbol{x}_{i}\boldsymbol{w}{}^{\top}\boldsymbol{x}_{i}\boldsymbol{x}_{i}{}^{\top}= & d\mathrm{diag}\{\|\boldsymbol{w}\|_{2}^{2}\}-\mathrm{diag}\{\boldsymbol{w}\circ\boldsymbol{w}\}+2d\boldsymbol{w}\boldsymbol{w}{}^{\top}
\end{align*}
\[
\|E\boldsymbol{x}_{i}\boldsymbol{x}_{i}{}^{\top}\boldsymbol{w}\boldsymbol{x}_{i}{}^{\top}\boldsymbol{x}_{i}\boldsymbol{w}{}^{\top}\boldsymbol{x}_{i}\boldsymbol{x}_{i}{}^{\top}\|_{2}\leq4d\|\boldsymbol{w}\|_{2}^{2}
\]
 Using matrix Bernstein's inequality,
\begin{align*}
\|\frac{1}{n}\sum_{i=1}^{n}\boldsymbol{x}_{i}\boldsymbol{x}_{i}{}^{\top}\boldsymbol{w}\boldsymbol{x}_{i}{}^{\top}\|_{2}\leq & \frac{4}{3}\frac{4\|\boldsymbol{w}\|_{2}\sqrt{2\log(2/\eta)}d}{n}\log(2d/\eta)\\
 & +3\sqrt{2\frac{4d\|\boldsymbol{w}\|_{2}^{2}}{n}\log(2d/\eta)}\\
\leq & C\|\boldsymbol{w}\|_{2}\sqrt{\frac{d}{n}}
\end{align*}
 Denote $\delta=C\sqrt{d/n}$, when $n\geq Cd/\delta^{2}$, $d\geq8\log(2n/\eta)$,
\[
\|\frac{1}{n}\sum_{i=1}^{n}\boldsymbol{x}_{i}\boldsymbol{x}_{i}{}^{\top}\boldsymbol{w}\boldsymbol{x}_{i}{}^{\top}\|_{2}\leq\|\boldsymbol{w}\|_{2}\delta
\]

\subsection{Proof of Lemma \ref{lem:concentation-I-XXT}}

According to Corollary \ref{cor:matrix-covariance-concentration},
when $d\geq8\log(2n/\eta)$,
\[
\|\boldsymbol{x}_{i}\|^{2}\leq2d
\]
 Therefore, with a probability at least $1-\eta$,
\begin{align*}
\|I-\frac{1}{n}XX{}^{\top}\|_{2}\leq & 9\epsilon\sqrt{\log(2d/\eta)/n}
\end{align*}
 for $n\geq2d/\epsilon^{2}$. Denote $\delta=9\epsilon\sqrt{\log(2d/\eta)/n}$,
then when $n\geq Cd/\delta^{2}$,
\begin{align*}
\|I-\frac{1}{n}XX{}^{\top}\|_{2}\leq & \delta\ .
\end{align*}

\subsection{Proof of Lemma \ref{lem:abc-recursive}}

To derive $\alpha_{t+1}$ ,
\begin{align*}
\|U_{\perp}^{*}{}^{\top}[M^{*}+O(2\delta\epsilon_{t})]U^{(t)}\|_{2}\leq & \|U_{\perp}^{*}{}^{\top}M^{*}U^{(t)}\|_{2}+2\delta\epsilon_{t}\\
\leq & 2\delta\epsilon_{t}
\end{align*}
\begin{align*}
\sigma_{k}\{U^{*}{}^{\top}[M^{*}+O(2\delta\epsilon_{t})]U^{(t)}\}\geq & U^{*}{}^{\top}M^{*}U^{(t)}-2\delta\epsilon_{t}\\
\geq & \sigma_{k}^{*}\sigma_{k}\{U^{*}{}^{\top}U^{(t)}\}-2\delta\epsilon_{t}\\
= & \sigma_{k}^{*}\cos\theta_{t}-2\delta\epsilon_{t}
\end{align*}
\begin{align*}
\alpha_{t+1}=\tan\theta_{t+1}= & \frac{\|U_{\perp}^{*}{}^{\top}[M^{*}+O(2\delta\epsilon_{t})]U^{(t)}\|_{2}}{\sigma_{k}\{U^{*}{}^{\top}[M^{*}+O(2\delta\epsilon_{t})]U^{(t)}\}}\\
\leq & \frac{2\delta\epsilon_{t}}{\sigma_{k}^{*}\cos\theta_{t}-2\delta\epsilon_{t}}\ .
\end{align*}
 According to the assumption, $\cos\theta_{t}\geq\frac{1}{\sqrt{5}},\ 2\delta\epsilon_{t}\leq\frac{1}{2\sqrt{5}}\sigma_{k}^{*}$,
therefore
\begin{align*}
\alpha_{t+1}\leq & 2\sqrt{5}\epsilon_{t}/\sigma_{k}^{*}=4\sqrt{5}\delta(\beta_{t}+\gamma_{t})/\sigma_{k}^{*}\ .
\end{align*}

To derive $\gamma_{t+1}$,
\begin{align*}
\gamma_{t+1}= & \|M^{*}-M^{(t+1)}\|_{2}\\
= & \|M^{*}-(U^{(t+1)}U^{(t+1)}{}^{\top}(H_{1}^{(t)}-H_{2}^{(t)}+M^{(t)}){}^{\top})\|_{2}\\
= & \|M^{*}-U^{(t+1)}U^{(t+1)}{}^{\top}(M^{*}+O(2\delta(\gamma_{t}+\beta_{t}))){}^{\top}\|_{2}\\
= & \|(I-U^{(t+1)}U^{(t+1)}{}^{\top})M^{*}+U^{(t+1)}U^{(t+1)}{}^{\top}O(2\delta(\gamma_{t}+\beta_{t}))){}^{\top}\|_{2}\\
\leq & \|(I-U^{(t+1)}U^{(t+1)}{}^{\top})M^{*}\|_{2}+O(2\delta(\gamma_{t}+\beta_{t}))\\
\leq & \tan\theta_{t+1}\|M^{*}\|_{2}+2\delta(\gamma_{t}+\beta_{t})\\
= & \alpha_{t+1}\|M^{*}\|_{2}+2\delta(\gamma_{t}+\beta_{t})\ .
\end{align*}

The recursive inequality of $\beta_{t}$ is trivial.

\section{Proof of Theorem \ref{thm:shifted-CI-RIP}}
\begin{proof}
Denote $\sigma_{1}=\|M\|_{2}$. Define random matrix
\begin{align*}
B_{i}= & \boldsymbol{x}_{i}\boldsymbol{x}_{i}{}^{\top}M\boldsymbol{x}_{i}\boldsymbol{x}_{i}{}^{\top}\ .
\end{align*}
 It is easy to check that
\begin{align*}
EB_{i}= & 2M+\mathrm{tr}(M)I\ .
\end{align*}
\begin{align*}
\|B_{i}-EB_{i}\|_{2}= & \|\boldsymbol{x}_{i}\boldsymbol{x}_{i}{}^{\top}M\boldsymbol{x}_{i}\boldsymbol{x}_{i}{}^{\top}-2M-\mathrm{tr}(M)I\|_{2}\\
\leq & \|\boldsymbol{x}_{i}\boldsymbol{x}_{i}{}^{\top}M\boldsymbol{x}_{i}\boldsymbol{x}_{i}{}^{\top}\|_{2}+\|2M-\mathrm{tr}(M)I\|_{2}\\
\leq & \|\boldsymbol{x}_{i}\boldsymbol{x}_{i}{}^{\top}M\boldsymbol{x}_{i}\boldsymbol{x}_{i}{}^{\top}\|_{2}+2\|M\|_{2}+|\mathrm{tr}(M)|\ .
\end{align*}
 According to Lemma \ref{lem:concentration-of-tr(M)}, with a probability
at least $1-\eta_{2}$, for any $i\in\{1,\cdots,n\}$,
\begin{align*}
|\boldsymbol{x}_{i}{}^{\top}M\boldsymbol{x}_{i}|\leq & |\mathrm{tr}(M)|+2\|M\|_{F}\sqrt{2\log(2n/\eta_{2})}\triangleq c_{1}\ .
\end{align*}
 Therefore we have, with a probability at least $(1-\eta_{1})(1-\eta_{2})$,
\begin{align*}
\|B_{i}-EB_{i}\|_{2}\leq & \|\boldsymbol{x}_{i}\|_{2}^{2}\ |\boldsymbol{x}_{i}{}^{\top}M\boldsymbol{x}_{i}|+2\|M\|_{2}+|\mathrm{tr}(M)|\\
\leq & 2d\log(2n/\eta_{1})|\mathrm{tr}(M)|+2\|M\|_{F}\sqrt{2\log(2n/\eta_{2})}+2\|M\|_{2}+|\mathrm{tr}(M)|\\
\leq & Cdk\sigma_{1}\ .
\end{align*}

Next we need to bound 
\begin{align*}
\|E(B_{i}-EB_{i})(B_{i}-EB_{i}){}^{\top}\|_{2} & =\|E(B_{i}^{2})-(EB_{i})^{2}\|_{2}\leq\|E(B_{i}^{2})\|_{2}+\|EB_{i}\|_{2}^{2}\\
 & \leq\|E(B_{i}^{2})\|_{2}+2|\mathrm{tr}(M)|^{2}+2\|M\|_{2}^{2}
\end{align*}
 To bound $\|E(B_{i}^{2})\|_{2}$, denote the eigenvalue decomposition
of $M$ as 
\[
M=U\Lambda U{}^{\top}=U\mathrm{diag}(\lambda_{1},\lambda_{2},\cdots,\lambda_{k})U{}^{\top}
\]
 Let $U_{\perp}$ be the complementary basis matrix of $U$. Define
random variables $\boldsymbol{u}_{i}\triangleq U{}^{\top}\boldsymbol{x}_{i}$,
$\boldsymbol{v}_{i}\triangleq U_{\perp}{}^{\top}\boldsymbol{x}_{i}$.
Since $\boldsymbol{x}_{i}$ are standard random Gaussian, $\boldsymbol{u}$
and $\boldsymbol{v}$ should also be jointly random Gaussian thus
independent.
\begin{align*}
\|E(B_{i}^{2})\|_{2}= & \|E(\boldsymbol{x}_{i}\boldsymbol{x}_{i}{}^{\top}M\boldsymbol{x}_{i}\boldsymbol{x}_{i}{}^{\top}\boldsymbol{x}_{i}\boldsymbol{x}_{i}{}^{\top}M\boldsymbol{x}_{i}\boldsymbol{x}_{i}{}^{\top})\|_{2}\\
= & \|E(\left[\begin{array}{c}
\boldsymbol{u}_{i}\\
\boldsymbol{v}_{i}
\end{array}\right]\boldsymbol{u}{}_{i}^{\top}\Lambda\boldsymbol{u}_{i}(\|\boldsymbol{u}_{i}\|_{2}^{2}+\|\boldsymbol{v}_{i}\|_{2}^{2})\boldsymbol{u}_{i}{}^{\top}\Lambda\boldsymbol{u}_{i}\left[\begin{array}{c}
\boldsymbol{u}_{i}\\
\boldsymbol{v}_{i}
\end{array}\right]{}^{\top})\|_{2}\\
\leq & \|E(\boldsymbol{u}_{i}\boldsymbol{u}{}_{i}^{\top}\Lambda\boldsymbol{u}_{i}(\|\boldsymbol{u}_{i}\|_{2}^{2}+\|\boldsymbol{v}_{i}\|_{2}^{2})\boldsymbol{u}_{i}{}^{\top}\Lambda\boldsymbol{u}_{i}\boldsymbol{u}_{i}{}^{\top}\|_{2}\\
 & +2\|E(\boldsymbol{u}_{i}\boldsymbol{u}{}_{i}^{\top}\Lambda\boldsymbol{u}_{i}(\|\boldsymbol{u}_{i}\|_{2}^{2}+\|\boldsymbol{v}_{i}\|_{2}^{2})\boldsymbol{u}_{i}{}^{\top}\Lambda\boldsymbol{u}_{i}\boldsymbol{v}_{i}{}^{\top}\|_{2}\\
 & +\|E(\boldsymbol{v}_{i}\boldsymbol{u}{}_{i}^{\top}\Lambda\boldsymbol{u}_{i}(\|\boldsymbol{u}_{i}\|_{2}^{2}+\|\boldsymbol{v}_{i}\|_{2}^{2})\boldsymbol{u}_{i}{}^{\top}\Lambda\boldsymbol{u}_{i}\boldsymbol{v}_{i}{}^{\top}\|_{2}\\
\leq & \|E(\boldsymbol{u}_{i}\boldsymbol{u}{}_{i}^{\top}\Lambda\boldsymbol{u}_{i}\|\boldsymbol{u}_{i}\|_{2}^{2}\boldsymbol{u}_{i}{}^{\top}\Lambda\boldsymbol{u}_{i}\boldsymbol{u}_{i}{}^{\top})\|_{2}+\|E(\boldsymbol{u}_{i}\boldsymbol{u}{}_{i}^{\top}\Lambda\boldsymbol{u}_{i}\|\boldsymbol{v}_{i}\|_{2}^{2}\boldsymbol{u}_{i}{}^{\top}\Lambda\boldsymbol{u}_{i}\boldsymbol{u}_{i}{}^{\top})\|_{2}\\
 & +2\|E(\boldsymbol{u}_{i}\boldsymbol{u}{}_{i}^{\top}\Lambda\boldsymbol{u}_{i}\|\boldsymbol{u}_{i}\|_{2}^{2}\boldsymbol{u}_{i}{}^{\top}\Lambda\boldsymbol{u}_{i}\boldsymbol{v}_{i}{}^{\top})\|_{2}+2\|E(\boldsymbol{u}_{i}\boldsymbol{u}{}_{i}^{\top}\Lambda\boldsymbol{u}_{i}\|\boldsymbol{v}_{i}\|_{2}^{2}\boldsymbol{u}_{i}{}^{\top}\Lambda\boldsymbol{u}_{i}\boldsymbol{v}_{i}{}^{\top})\|_{2}\\
 & +\|E(\boldsymbol{v}_{i}\boldsymbol{u}{}_{i}^{\top}\Lambda\boldsymbol{u}_{i}\|\boldsymbol{u}_{i}\|_{2}^{2}\boldsymbol{u}_{i}{}^{\top}\Lambda\boldsymbol{u}_{i}\boldsymbol{v}_{i}{}^{\top})\|_{2}+\|E(\boldsymbol{v}_{i}\boldsymbol{u}{}_{i}^{\top}\Lambda\boldsymbol{u}_{i}\|\boldsymbol{v}_{i}\|_{2}^{2}\boldsymbol{u}_{i}{}^{\top}\Lambda\boldsymbol{u}_{i}\boldsymbol{v}_{i}{}^{\top})\|_{2}\ .
\end{align*}
 Let us bound the above 6 terms respectively. Recall that with a probability
at least $1-\eta_{2}$, 
\begin{align*}
|\boldsymbol{u}{}_{i}^{\top}\Lambda\boldsymbol{u}_{i}|= & |\boldsymbol{x}_{i}{}^{\top}M\boldsymbol{x}_{i}|\leq c_{1}\ .
\end{align*}
 With a probability at least $1-\eta_{3},$ for any $i\in\{1,\cdots,n\}$,
$\|\boldsymbol{u}_{i}\|_{2}\leq2\sqrt{k\log(4n/\eta_{3})}$,$\|\boldsymbol{v}_{i}\|_{2}\leq2\sqrt{d\log(4n/\eta_{3})}$.
Then:

\begin{align*}
 & \|E(\boldsymbol{u}_{i}\boldsymbol{u}{}_{i}^{\top}\Lambda\boldsymbol{u}_{i}\|\boldsymbol{u}_{i}\|_{2}^{2}\boldsymbol{u}_{i}{}^{\top}\Lambda\boldsymbol{u}_{i}\boldsymbol{u}_{i}{}^{\top})\|_{2}\\
= & \|E\{\left((\boldsymbol{u}{}_{i}^{\top}\Lambda\boldsymbol{u}_{i})^{2}\|\boldsymbol{u}_{i}\|_{2}^{2}\right)\boldsymbol{u}_{i}\boldsymbol{u}_{i}{}^{\top}\}\|_{2}\\
\leq & (\boldsymbol{u}{}_{i}^{\top}\Lambda\boldsymbol{u}_{i})^{2}\|\boldsymbol{u}_{i}\|_{2}^{4}\\
\leq & 32c_{1}^{2}k^{2}\log^{2}(2n/\eta_{3})\ .
\end{align*}
\begin{align*}
 & \|E(\boldsymbol{u}_{i}\boldsymbol{u}{}_{i}^{\top}\Lambda\boldsymbol{u}_{i}\|\boldsymbol{v}_{i}\|_{2}^{2}\boldsymbol{u}_{i}{}^{\top}\Lambda\boldsymbol{u}_{i}\boldsymbol{u}_{i}{}^{\top})\|_{2}\\
= & \|E(\|\boldsymbol{v}_{i}\|_{2}^{2})E(\boldsymbol{u}_{i}\boldsymbol{u}{}_{i}^{\top}\Lambda\boldsymbol{u}_{i}\boldsymbol{u}_{i}{}^{\top}\Lambda\boldsymbol{u}_{i}\boldsymbol{u}_{i}{}^{\top})\|_{2}\\
\leq & 4d\log(4n/\eta_{3})\|E(\boldsymbol{u}_{i}\boldsymbol{u}{}_{i}^{\top}\Lambda\boldsymbol{u}_{i}\boldsymbol{u}_{i}{}^{\top}\Lambda\boldsymbol{u}_{i}\boldsymbol{u}_{i}{}^{\top})\|_{2}\\
\leq & 4d\log(4n/\eta_{3})\|\boldsymbol{u}_{i}\|_{2}^{2}(\boldsymbol{u}{}_{i}^{\top}\Lambda\boldsymbol{u}_{i})^{2}\\
\leq & 4d\log(4n/\eta_{3})c_{1}^{2}(4k\log(4n/\eta_{3}))\ .
\end{align*}
 
\begin{align*}
 & 2\|E(\boldsymbol{u}_{i}\boldsymbol{u}{}_{i}^{\top}\Lambda\boldsymbol{u}_{i}\|\boldsymbol{u}_{i}\|_{2}^{2}\boldsymbol{u}_{i}{}^{\top}\Lambda\boldsymbol{u}_{i}\boldsymbol{v}_{i}{}^{\top})\|_{2}\\
= & 2\|E(\boldsymbol{u}_{i}\boldsymbol{u}{}_{i}^{\top}\Lambda\boldsymbol{u}_{i}\|\boldsymbol{u}_{i}\|_{2}^{2}\boldsymbol{u}_{i}{}^{\top}\Lambda\boldsymbol{u}_{i})E(\boldsymbol{v}_{i}{}^{\top})\|_{2}=0
\end{align*}
\begin{align*}
 & 2\|E(\boldsymbol{u}_{i}\boldsymbol{u}{}_{i}^{\top}\Lambda\boldsymbol{u}_{i}\|\boldsymbol{v}_{i}\|_{2}^{2}\boldsymbol{u}_{i}{}^{\top}\Lambda\boldsymbol{u}_{i}\boldsymbol{v}_{i}{}^{\top})\|_{2}\\
= & 2\|E(\boldsymbol{u}_{i}(\boldsymbol{u}{}_{i}^{\top}\Lambda\boldsymbol{u}_{i})^{2})E(\|\boldsymbol{v}_{i}\|_{2}^{2}\boldsymbol{v}_{i}{}^{\top})\|_{2}\\
= & 2\|E(\boldsymbol{u}_{i}(\boldsymbol{u}{}_{i}^{\top}\Lambda\boldsymbol{u}_{i})^{2})E(\boldsymbol{v}_{i}{}^{\top}\boldsymbol{v}_{i}\boldsymbol{v}_{i}{}^{\top})\|_{2}=0
\end{align*}
\begin{align*}
 & \|E(\boldsymbol{v}_{i}\boldsymbol{u}{}_{i}^{\top}\Lambda\boldsymbol{u}_{i}\|\boldsymbol{u}_{i}\|_{2}^{2}\boldsymbol{u}_{i}{}^{\top}\Lambda\boldsymbol{u}_{i}\boldsymbol{v}_{i}{}^{\top})\|_{2}\\
= & \|E(\boldsymbol{u}{}_{i}^{\top}\Lambda\boldsymbol{u}_{i}\|\boldsymbol{u}_{i}\|_{2}^{2}\boldsymbol{u}_{i}{}^{\top}\Lambda\boldsymbol{u}_{i})E(\boldsymbol{v}_{i}\boldsymbol{v}_{i}{}^{\top})\|_{2}\\
= & \|E(\boldsymbol{u}{}_{i}^{\top}\Lambda\boldsymbol{u}_{i}\|\boldsymbol{u}_{i}\|_{2}^{2}\boldsymbol{u}_{i}{}^{\top}\Lambda\boldsymbol{u}_{i})\|_{2}\\
\leq & (\boldsymbol{u}{}_{i}^{\top}\Lambda\boldsymbol{u}_{i})^{2}\|\boldsymbol{u}_{i}\|_{2}^{2}\\
\leq & 4c_{1}^{2}k\log(4n/\eta_{3})\ .
\end{align*}
\begin{align*}
 & \|E(\boldsymbol{v}_{i}\boldsymbol{u}{}_{i}^{\top}\Lambda\boldsymbol{u}_{i}\|\boldsymbol{v}_{i}\|_{2}^{2}\boldsymbol{u}_{i}{}^{\top}\Lambda\boldsymbol{u}_{i}\boldsymbol{v}_{i}{}^{\top})\|_{2}\\
= & \|E\{(\boldsymbol{u}{}_{i}^{\top}\Lambda\boldsymbol{u}_{i})^{2}\}E(\boldsymbol{v}_{i}\|\boldsymbol{v}_{i}\|_{2}^{2}\boldsymbol{v}_{i}{}^{\top})\|_{2}\\
= & \|E\{(\boldsymbol{u}{}_{i}^{\top}\Lambda\boldsymbol{u}_{i})^{2}\}(d+2)I\|_{2}\\
\leq & (d+2)(\boldsymbol{u}{}_{i}^{\top}\Lambda\boldsymbol{u}_{i})^{2}\\
\leq & c_{1}^{2}(d+2)
\end{align*}
 Add all above together, we have 
\begin{align*}
\|E(B_{i}^{2})\|_{2}\leq & 32c_{1}^{2}k^{2}\log^{2}(2n/\eta_{3})+4d\log(4n/\eta_{3})c_{1}^{2}(4k\log(4n/\eta_{3}))\\
 & +4c_{1}^{2}k\log(4n/\eta_{3})+c_{1}^{2}(d+2)\\
\leq & Ck^{3}d\sigma_{1}\ .
\end{align*}

Apply matrix Bernsterin's inequality, the proof is completed.
\end{proof}

\section{Proof of Lemma \ref{thm:the-solution-of-noisy-recursive-inequality}}

We assume that $n\geq Ck^{3}d/\delta^{2}$ . 

To prove Eq. (\ref{eq:noisy-AAM-concentration})

\begin{align*}
 & \|\frac{1}{2n}\mathcal{A}'\mathcal{A}(M^{*}-M)-\frac{1}{2}\mathrm{tr}(M_{k}^{*}-M)I-(M_{k}^{*}-M)\|_{2}\\
\leq & \|\frac{1}{2n}\mathcal{A}'\mathcal{A}(M_{k}^{*}-M)-\frac{1}{2}\mathrm{tr}(M_{k}^{*}-M)I-(M_{k}^{*}-M)\|_{2}+\|\frac{1}{2n}\mathcal{A}'\mathcal{A}(M_{\perp}^{*})\|_{2}\\
\leq & \|\frac{1}{2n}\mathcal{A}'\mathcal{A}(M_{\perp}^{*})\|_{2}+\delta\|M_{k}^{*}-M\|_{2}\ .
\end{align*}
 The last inequality is because of Theorem \ref{thm:shifted-CI-RIP}.
To bound the first term in the last inequality, define random matrix
\begin{align*}
B_{i} & =\boldsymbol{x}_{i}\boldsymbol{x}_{i}{}^{\top}M_{\perp}^{*}\boldsymbol{x}_{i}\boldsymbol{x}_{i}{}^{\top}
\end{align*}
 As proved in Theorem \ref{thm:shifted-CI-RIP}, $EB_{i}=2M_{\perp}^{*}+\mathrm{tr}(M_{\perp}^{*})I$.
\begin{align*}
\|(B_{i}-EB_{i})\|_{2}= & \|\boldsymbol{x}_{i}\boldsymbol{x}_{i}{}^{\top}M_{\perp}^{*}\boldsymbol{x}_{i}\boldsymbol{x}_{i}{}^{\top}-2M_{\perp}^{*}+\mathrm{tr}(M_{\perp}^{*})I\|_{2}\\
\leq & \|\boldsymbol{x}_{i}\boldsymbol{x}_{i}{}^{\top}M_{\perp}^{*}\boldsymbol{x}_{i}\boldsymbol{x}_{i}{}^{\top}\|_{2}+2\|M_{\perp}^{*}\|_{2}+\|\mathrm{tr}(M_{\perp}^{*})I\|_{2}\\
= & \|\boldsymbol{x}_{i}\boldsymbol{x}_{i}{}^{\top}M_{\perp}^{*}\boldsymbol{x}_{i}\boldsymbol{x}_{i}{}^{\top}\|_{2}+2\sigma_{k+1}^{*}+|\mathrm{tr}(M_{\perp}^{*})|
\end{align*}
 While
\begin{align*}
\|\boldsymbol{x}_{i}\boldsymbol{x}_{i}{}^{\top}M_{\perp}^{*}\boldsymbol{x}_{i}\boldsymbol{x}_{i}{}^{\top}\|_{2}\leq & \|M_{\perp}^{*}\|_{2}\|\boldsymbol{x}_{i}\|_{2}^{4}\\
\leq & \sigma_{k+1}^{*}(d+2\sqrt{2d\log(2n/\eta)})^{2}\\
\leq & Cd^{2}\sigma_{k+1}^{*}
\end{align*}
 Applying matrix Bernstein's inequality, with a probability at least
$1-\eta$, we have
\begin{align*}
\|\frac{1}{n}\sum_{i=1}^{n}(B_{i}-EB_{i})\|_{2}\leq & C\sigma_{k+1}^{*}d^{2}/\sqrt{n}\ .
\end{align*}
 Therefore
\begin{align*}
\|\frac{1}{2n}\mathcal{A}'\mathcal{A}(M^{*}-M)-\frac{1}{2}\mathrm{tr}(M_{k}^{*}-M)I-(M_{k}^{*}-M)\|_{2}\leq & \delta\|M_{k}^{*}-M\|_{2}+C\sigma_{k+1}^{*2}d^{4}/\sqrt{n}\ .
\end{align*}

The other inequalities can be similarly proved.

\section{Proof of Lemma \ref{lem:noisy-recursive-inequality}}

First we bound $\alpha_{t+1}$. According to assumption, when
\begin{align*}
 & 2(\delta\epsilon_{t}+r)\leq\frac{\sigma_{k}^{*}-\sigma_{k+1}^{*}}{2\sigma_{k}^{*}}\sigma_{k}^{*}/\sqrt{5}
\end{align*}
 we have 
\begin{align*}
\alpha_{t+1}\leq & \frac{\sigma_{k+1}^{*}\sin\theta_{t}+2(\delta\epsilon_{t}+r)}{\sigma_{k}^{*}\cos\theta_{t}-2(\delta\epsilon_{t}+r)}\\
\leq & \frac{2\sigma_{k}^{*}}{\sigma_{k}^{*}+\sigma_{k+1}^{*}}\frac{\sigma_{k+1}^{*}\sin\theta_{t}+2(\delta\epsilon_{t}+r)}{\sigma_{k}^{*}\cos\theta_{t}}\\
\leq & \frac{2\sigma_{k+1}^{*}}{\sigma_{k}^{*}+\sigma_{k+1}^{*}}\tan\theta_{t}+\frac{2}{\sigma_{k}^{*}+\sigma_{k+1}^{*}}\frac{2(\delta\epsilon_{t}+r)}{\cos\theta_{t}}\\
\leq & \frac{2\sigma_{k+1}^{*}}{\sigma_{k}^{*}+\sigma_{k+1}^{*}}\tan\theta_{t}+\frac{4\sqrt{5}}{\sigma_{k}^{*}+\sigma_{k+1}^{*}}(\delta\epsilon_{t}+r)\\
\leq & \rho\alpha_{t}+\frac{4\sqrt{5}}{\sigma_{k}^{*}+\sigma_{k+1}^{*}}\delta\epsilon_{t}+\frac{4\sqrt{5}}{\sigma_{k}^{*}+\sigma_{k+1}^{*}}r\ .
\end{align*}

To bound $\beta_{t+1}$. Clearly $\beta_{t+1}\leq\delta\epsilon_{t}+r$.

To bound $\gamma_{t+1}$, following the noise-free case,
\begin{align*}
\gamma_{t+1}\leq & \alpha_{t+1}\|M^{*}\|_{2}+2\delta\epsilon_{t}+2r\ .
\end{align*}

\section{Proof of Lemma \ref{thm:the-solution-of-noisy-recursive-inequality}}

Abbreviate
\begin{align*}
c= & \frac{4\sqrt{5}}{\sigma_{k}^{*}+\sigma_{k+1}^{*}}
\end{align*}
 Then
\begin{align*}
\alpha_{t+1}\leq & \rho\alpha_{t}+c\delta\epsilon_{t}+cr\ .
\end{align*}
According to Lemma \ref{lem:noisy-recursive-inequality}, 
\begin{align*}
\beta_{t+1}+\gamma_{t+1}\leq & \delta\epsilon_{t}+r+\alpha_{t+1}\|M^{*}\|_{2}+2\delta\epsilon_{t}+2r\\
= & \sigma_{1}^{*}\alpha_{t+1}+3\delta\epsilon_{t}+3r\\
\leq & \sigma_{1}^{*}(\rho\alpha_{t}+c\delta\epsilon_{t}+cr)+3\delta\epsilon_{t}+3r\\
= & \rho\sigma_{1}^{*}\alpha_{t}+(\sigma_{1}^{*}c+3)\delta\epsilon_{t}+(\sigma_{1}^{*}c+3)r
\end{align*}
 Therefore, abbreviate $b\triangleq(\sigma_{1}^{*}c+3)$,
\begin{align*}
\begin{cases}
\alpha_{t+1}\leq\rho\alpha_{t}+c\delta\epsilon_{t}+cr\\
\epsilon_{t+1}\leq\rho\sigma_{1}^{*}\alpha_{t}+b\delta\epsilon_{t}+br
\end{cases}
\end{align*}
 define
\begin{align*}
f_{t}= & \alpha_{t}+2c\delta\epsilon_{t}
\end{align*}
 
\begin{align*}
f_{t+1}= & a_{t+1}+2c\delta\epsilon_{t+1}\\
\leq & \rho\alpha_{t}+c\delta\epsilon_{t}+cr+2c\delta(\rho\sigma_{1}^{*}\alpha_{t}+b\delta\epsilon_{t}+br)\\
= & \rho\alpha_{t}+c\delta\epsilon_{t}+cr+2c\delta\rho\sigma_{1}^{*}\alpha_{t}+2c\delta b\delta\epsilon_{t}+2c\delta br\\
= & (\rho+2c\delta\rho\sigma_{1}^{*})\alpha_{t}+(c+2c\delta b)\delta\epsilon_{t}+(1+2\delta b)cr
\end{align*}
When
\begin{align*}
 & \delta\leq\frac{1-\rho}{4\rho\sigma_{1}^{*}c}\\
\Rightarrow & \rho+2c\delta\rho\sigma_{1}^{*}\leq\frac{1+\rho}{2}
\end{align*}
 And when
\begin{align*}
\Rightarrow & \delta\leq\frac{\rho}{2b}\\
\Rightarrow & 2\delta b\leq\rho\\
\Rightarrow & 2c\delta b\leq\rho c\\
\Rightarrow & c+2c\delta b\leq(1+\rho)c\\
\Rightarrow & c+2c\delta b\leq\frac{1+\rho}{2}2c
\end{align*}
 Then abbreviate $R\triangleq(c+2c\delta b)\delta\epsilon_{t}+(1+2\delta b)cr$
we have
\begin{align*}
 & f_{t+1}\leq\frac{1+\rho}{2}f_{t}+(1+2\delta b)cr\leq\frac{1+\rho}{2}f_{t}+(1+\rho)cr
\end{align*}
 Abbreviate $q=(1+\rho)/2$,
\begin{align*}
f_{t}\leq & \frac{(1+\rho)cr}{1-q}+q^{t}(f_{0}-\frac{(1+\rho)cr}{1-q})
\end{align*}
 To ensure $\alpha_{t+1}\leq2$, we require
\begin{align*}
 & f_{0}\leq2\\
\Leftarrow & \alpha_{0}+2c\delta\epsilon_{0}\leq2
\end{align*}
 According to Lemma \ref{lem:init-singular-vector-perturbation},
\begin{align*}
\alpha_{0}\leq & \frac{2}{\sigma_{k}^{*}-\sigma_{k+1}^{*}}2(\delta\epsilon_{0}+r)=\frac{4}{\sigma_{k}^{*}-\sigma_{k+1}^{*}}(\delta\epsilon_{0}+r)
\end{align*}
 
\begin{align*}
 & \alpha_{0}+2c\delta\epsilon_{0}\leq2\\
\Leftarrow & \frac{4}{\sigma_{k}^{*}-\sigma_{k+1}^{*}}(\delta\epsilon_{0}+r)+2c\delta\epsilon_{0}\leq2\\
\Leftarrow & (4+2c(\sigma_{k}^{*}-\sigma_{k+1}^{*}))\delta\epsilon_{0}+4r\leq2(\sigma_{k}^{*}-\sigma_{k+1}^{*})\\
\Leftarrow & (2+c(\sigma_{k}^{*}-\sigma_{k+1}^{*}))\delta\epsilon_{0}+r\leq(\sigma_{k}^{*}-\sigma_{k+1}^{*})
\end{align*}

In summary, 
\begin{align*}
\alpha_{t}+2c\delta\epsilon_{t}\leq & q^{t}(f_{0}-\frac{(1+\rho)cr}{1-q})+\frac{(1+\rho)cr}{1-q}
\end{align*}
 provided
\begin{align*}
\delta\leq & \min\{\frac{1-\rho}{4\rho\sigma_{1}^{*}c},\frac{\rho}{2b}\}
\end{align*}
 and
\begin{gather*}
(2+c(\sigma_{k}^{*}-\sigma_{k+1}^{*}))\delta\epsilon_{0}+r\leq(\sigma_{k}^{*}-\sigma_{k+1}^{*})\\
4\sqrt{5}(\delta\max_{t}\epsilon_{t}+r)\leq\sigma_{k}^{*}-\sigma_{k+1}^{*}
\end{gather*}
 To ensure the last inequality,
\begin{align*}
\delta\max_{t}\epsilon_{t}\leq & f_{0}\leq\alpha_{0}+2c\delta\epsilon_{0}\leq\frac{4}{\sigma_{k}^{*}-\sigma_{k+1}^{*}}(\delta\epsilon_{0}+r)+2c\delta\epsilon_{0}\\
= & (\frac{4}{\sigma_{k}^{*}-\sigma_{k+1}^{*}}+2c)\delta\epsilon_{0}+\frac{4}{\sigma_{k}^{*}-\sigma_{k+1}^{*}}r
\end{align*}
 Therefore we need the condition
\begin{align*}
 & 4\sqrt{5}\left(\frac{4}{\sigma_{k}^{*}-\sigma_{k+1}^{*}}+2c\right)\delta\epsilon_{0}+4\sqrt{5}\left(\frac{4}{\sigma_{k}^{*}-\sigma_{k+1}^{*}}+1\right)r\leq\sigma_{k}^{*}-\sigma_{k+1}^{*}\\
\Leftarrow & 4\sqrt{5}\left(4+2c(\sigma_{k}^{*}-\sigma_{k+1}^{*})\right)\delta\epsilon_{0}+4\sqrt{5}\left(4+(\sigma_{k}^{*}-\sigma_{k+1}^{*})\right)r\leq(\sigma_{k}^{*}-\sigma_{k+1}^{*})^{2}
\end{align*}

\end{document}